\documentclass{applemlr}
\usepackage{amsmath,amssymb,amsfonts,amsthm,bm,mathtools}
\usepackage{microtype}
\newtheorem{proposition}{Proposition}

\newtheorem{lemma}{Lemma}
\newtheorem{remark}{Remark}
\usepackage{listings}
\usepackage{enumerate}
\usepackage{algorithm}
\usepackage{algpseudocode}
\usepackage{amsfonts}
\usepackage{amsthm}
\usepackage{newtxtt}
\usepackage{cleveref}
\usepackage{diagbox}
\usepackage{colortbl}
\usepackage{amssymb}
\usepackage{xspace}
\usepackage{wrapfig}
\usepackage{adjustbox}
\usepackage{tabularx}
\usepackage{booktabs}
\usepackage{mathtools}
\usepackage{tikz}
\usepackage{enumitem}
\usepackage{silence}
\usepackage{dsfont}
\usepackage[table]{xcolor}
\usepackage[dvipsnames]{xcolor}
\usepackage{multirow}
\usepackage{makecell}
\usepackage{xfakebold}
\usepackage{amsmath,amsfonts,bm}

\newcommand{\Cat}{\mathrm{Categorical}}

\def\Figref#1{Figure~\ref{#1}}

\def\eqref#1{equation~\ref{#1}}
\def\1{\bm{1}}

\def\veps{{\boldsymbol{\epsilon}}}

\def\vzero{{\bm{0}}}

\def\vmu{{\bm{\mu}}}

\def\vm{{\bm{m}}}

\def\vw{{\bm{w}}}
\def\vx{{\bm{x}}}

\def\vz{{\bm{z}}}

\def\mI{{\bm{I}}}

\def\mQ{{\bm{Q}}}

\DeclareMathAlphabet{\mathsfit}{\encodingdefault}{\sfdefault}{m}{sl}
\SetMathAlphabet{\mathsfit}{bold}{\encodingdefault}{\sfdefault}{bx}{n}

\def\gL{{\mathcal{L}}}

\newcommand{\E}{\mathbb{E}}

\newcommand{\R}{\mathbb{R}}

\newcommand{\KL}{D_{\mathrm{KL}}}

\DeclareMathOperator*{\argmax}{arg\,max}

\definecolor{textgray}{HTML}{6E6E73}
\usetikzlibrary{positioning, calc}
\usetikzlibrary{decorations.pathmorphing}

\makeatletter
\patchcmd{\wrong@fontshape}{\@gobbletwo}{}{}{}
\makeatother
\numberwithin{equation}{section}
\tcbuselibrary{minted}
\setminted[python]{
  linenos,
  breaklines,
  fontsize=\footnotesize,
  xleftmargin=2em
}

\makeatletter
\AtBeginDocument{
  \urlstyle{sf}
  
}
\makeatother

\definecolor{light}{RGB}{125, 125, 125}
\crefname{tcb@cnt@pbox}{code}{code}
\Crefname{tcb@cnt@pbox}{Code}{Code}
\crefname{assumption}{assumption}{assumption}
\Crefname{assumption}{Assumption}{Assumptions}

\newtcolorbox[auto counter]{pbox}[2][]{
  colback=white,
  title=Code~\thetcbcounter: #2,
  #1,fonttitle=\sffamily,
  fontupper=\sffamily,
  arc=2pt,
  colframe=bgcolor,
  coltitle=fgcolor,
  colbacktitle=bgcolor,
  toptitle=0.25cm,
  bottomtitle=0.125cm
}

\makeatletter
\newcommand\applefootnote[1]{%
  \begingroup
  \renewcommand\thefootnote{}%
  \renewcommand\@makefntext[1]{\noindent##1}%
  \footnote{#1}%
  \addtocounter{footnote}{-1}%
  \endgroup
}
\makeatother

\definecolor{cverbbg}{gray}{0.90}

\usepackage{tcolorbox}

\definecolor{codegreen}{rgb}{0,0.3,0.6}
\definecolor{codegray}{rgb}{0.5,0.5,0.5}
\definecolor{codepurple}{rgb}{0.58,0,0.82}
\definecolor{backcolour}{rgb}{0.95,0.95,0.92}
\definecolor{orange}{rgb}{1,0.5,0}
\definecolor{mydarkblue}{rgb}{0,0.08,0.45}

\lstdefinestyle{mystyle}{
    backgroundcolor=\color{backcolour},   
    commentstyle=\color{codegreen},
    keywordstyle=\color{magenta},
    numberstyle=\tiny\color{codegray},
    stringstyle=\color{codepurple},
    basicstyle=\ttfamily\footnotesize,
    breakatwhitespace=false,         
    breaklines=true,                 
    captionpos=b,                    
    keepspaces=true,                 
    numbers=left,                    
    numbersep=5pt,                  
    showspaces=false,                
    showstringspaces=false,
    showtabs=false,                  
    tabsize=2
}
\title{Continuously Augmented Discrete Diffusion model for Categorical Generative Modeling}
\author{Huangjie Zheng}
\author[\dagger]{Shansan Gong}
\author{Ruixiang Zhang}
\author{Tianrong Chen}
\author{Jiatao Gu}
\author[\dagger]{Mingyuan Zhou}
\author{Navdeep Jaitly}
\author{Yizhe Zhang}

\affiliation{Apple}

\abstract{
Standard discrete diffusion models treat all unobserved states identically by mapping them to an absorbing \texttt{[MASK]} token. This creates an ``information void'' where semantic information that could be inferred from unmasked tokens is lost between denoising steps. 
We introduce \emph{Continuously Augmented Discrete Diffusion} (CADD), a framework that augments the discrete state space with a paired diffusion in a continuous latent space. This yields graded, gradually corrupted states in which masked tokens are represented by noisy yet informative latent vectors rather than collapsed ``information voids''. At each reverse step, CADD may leverage the continuous latent as a semantic hint to guide discrete denoising. The design is clean and compatible with existing discrete diffusion training. At sampling time, the strength and choice of estimator for the continuous latent vector enables a controlled trade-off between mode-coverage (generating diverse outputs) and mode-seeking (generating contextually precise outputs) behaviors. Empirically, we demonstrate CADD improves generative quality over mask-based diffusion across text generation, image synthesis, and code modeling, with consistent gains on both qualitative and quantitative metrics against strong discrete baselines.}

\metadata[Correspondence]{\sffamily Huangjie Zheng: \url{huangjie_zheng@apple.com}; Yizhe Zhang: \url{yizhe_zhang@apple.com}}
\date{\sffamily\today}

\begin{document}

\maketitle

\applefootnote{ \textcolor{textgray}{\sffamily $\dagger$ Work done when Shansan Gong and Mingyuan Zhou were at Apple. }}

\begin{figure}[h]
    \centering
    \includegraphics[width=\linewidth]{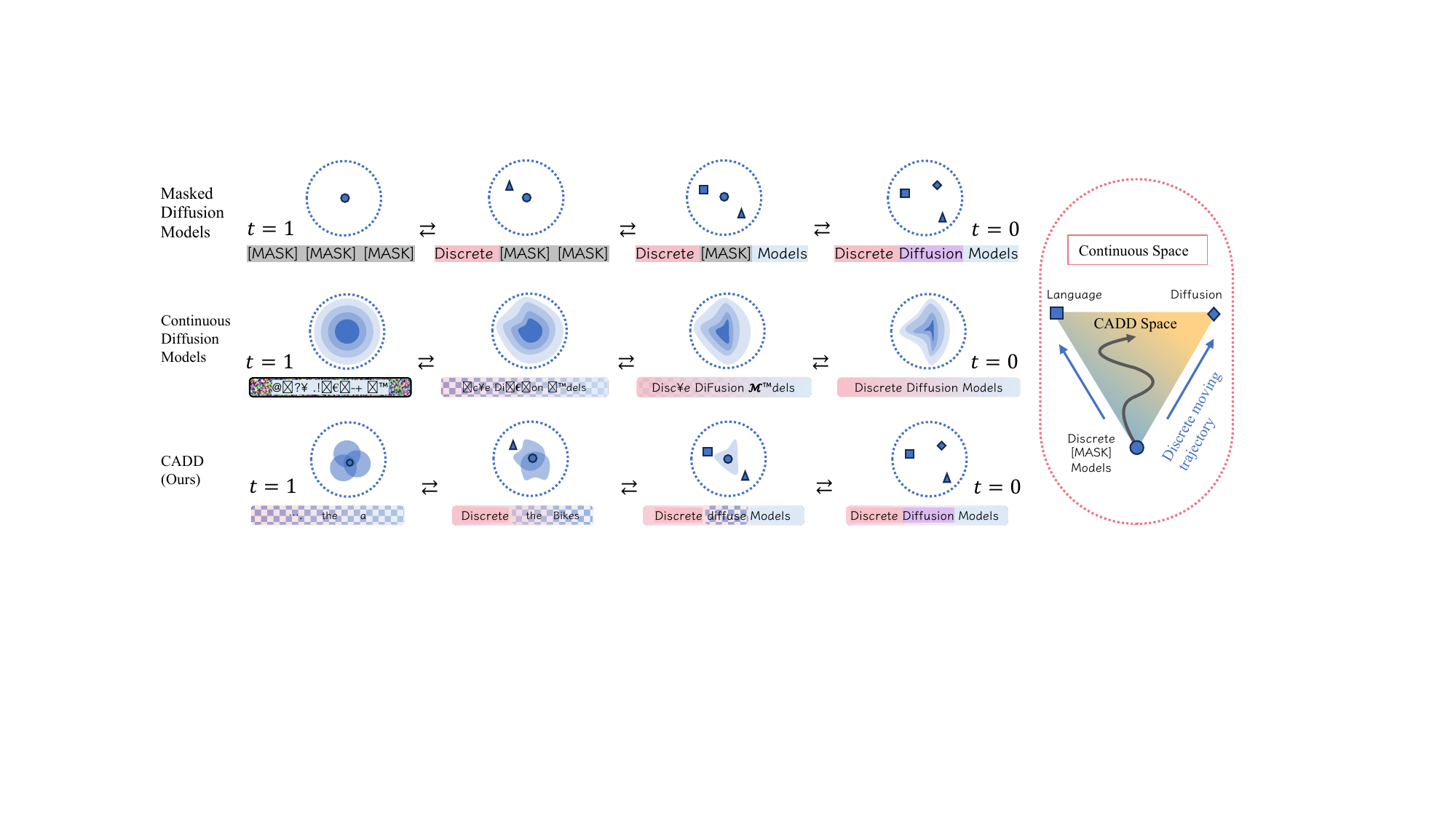}
    \caption{(\emph{Best view in color}) \textbf{Comparison of diffusion models across modeling spaces.} Masked diffusion uses \texttt{[MASK]} as noise and follows a single mask-to-token path, jumping from an absorbing state to token predictions. Continuous (Gaussian) diffusion evolves in the full embedding space, but intermediate latents often do not decode to valid tokens until the final step because the search space is large. CADD combines the \textit{stability} of masked diffusion with the \textit{flexibility} of continuous diffusion: discrete tokens anchor a context-consistent subspace, while the paired continuous latent allows smooth transitions among plausible token candidates, improving decoding at masked positions.}
    \label{fig:teaser}
\end{figure}

\section{Introduction}

Diffusion models have significantly advanced generative modeling tasks~\citep{dickstein2015diffusion,ho2020ddpm,song2021scorebased,dhariwal2021image,Karras2022edm}, particularly in image synthesis~\citep{saharia2022image, esser2024image, polyak2024moviegen, zheng2024learning,brooks2024video}. Recently, with rapid progress in discrete diffusion models~\citep{austin2021d3pm,hoogeboom2021multinomial,lou2024sedd}, diffusion models have become a powerful tool for discrete categorical data domains, such as text generation and code generation~\citep{gat2024discrete,gong2022diffuseq,gong2025diffucoder}.

Early work on Continuous Diffusion Models (CDMs) for categorical data maps tokens into a continuous space, applies Gaussian diffusion to the representations, and then rounds back to discrete symbols~\citep{li2022diffusion,dieleman2022continuous,han2022ssd,zhang2023planner,gulrajani2023likelihoodbased}. This route preserves smooth semantic signals and enables the use of established score-based methods. In parallel, Masked Diffusion Models (MDMs) have recently shown strong results for categorical data~\citep{shi2024md4,sahoo2024simple,nie2025llmdiff}: instead of adding noise in an embedding space, MDMs progressively mask tokens over time and learn to unmask them, yielding clear training signals via token-level cross-entropy.

Despite their respective successes, both approaches have  limitations, which are illustrated in~\Figref{fig:teaser}. (i) MDMs suffer from information loss due to their use of absorbing \texttt{[MASK]} state~\citep{chao2025mdmprime,wang2025remasking}. This design collapses all unobserved possibilities into one symbol, erasing information about how similar a corrupted position is to the original token, thus creating an ``information void''. This reduces the information available for resolving ambiguity and maintaining global semantic coherence. For example, as shown on the right of the figure, if a masked token could plausibly be ``Language'' or ``Diffusion'', the \texttt{[MASK]} representation offers no semantic clue to favor either option, forcing the model to make a hard choice without graded guidance.
(ii) While CDMs can represent semantic proximity, they face a different challenge known as ``over-smoothing''. Because the denoising process occurs entirely in continuous embedding space with discretization to tokens only at the end~\citep{gao2022empowering}, their continuous denoising objective can over-smooth token identities, making it difficult to make precise predictions without localized context—a problem known as ``rounding error'' \citep{li2022diffusion}.

To address these challenges, we propose \textbf{Continuously Augmented Discrete Diffusion (CADD)}, which combines the strengths of both CDMs and MDMs. CADD keeps the discrete masking process but augments a parallel continuous diffusion in continuous semantic embedding space. This means masked positions retain semantic information through noisy but informative latent vectors instead of becoming collapsed information voids. In the reverse process, the model uses the continuous latent as a soft semantic hint to guide token denoising at each step, while the discrete context constrains the latent dynamics locally. Returning to \Figref{fig:teaser}, the continuous manifold offers a graded path between candidates (``Language'' and ``Diffusion'', in this case), and the discrete neighborhood restricts the search space, allowing movement within the triangular region between hypotheses and enabling smooth transitions driven by the hints. In addressing the limits of both pure MDMs and CDMs, our contributions are:

\begin{enumerate}
\item \textit{Better token prediction with soft hints.} For masked positions, the continuous latent representations are corrupted in a smooth decay rather than an abrupt information loss, thus preserve graded proximity to the ground-truth token embedding, which reduces ambiguity and makes discrete prediction easier.

\item \textit{Diversity with multi-sample estimation.} At inference, one can resample the continuous latent (e.g., multiple latent draws per discrete state) to explore alternative yet valid choices for a token or span, which could lead to a complete view of plausible tokens, enhancing the diversity of generation.

\item \textit{Training and sampling remain simple.} CADD keeps standard cross-entropy for tokens and a standard diffusion loss for the continuous head. The sampler can alternate or jointly update the discrete and continuous states.

\item \textit{Parameter efficiency and efficient fine-tuning.} CADD requires no special architecture and can reuse the same backbone as an MDM. As a result, the number of learnable parameters matches prior MDMs, and there is no significant increase in compute cost in training. Together with simple training loss described above, this enables efficient fine-tuning of existing MDM checkpoints to obtain the benefits of CADD.
\end{enumerate}

\section{Related Work}

\paragraph{Discrete Diffusion Models}
Discrete diffusion models~\citep{hoogeboom2021multinomial, Zheng2023ARD, austin2021d3pm} operate by defining a Markov chain over the discrete token space, gradually diffusing the data with either uniform or absorbing transitions. Later, the model was unified and simplified to continuous-time masked diffusion models~\citep{campbell2022ctmc, lou2024sedd, shi2024md4, sahoo2024simple,zhang2025target}. Building on this, several recent works further scaled diffusion LMs to 7B parameters~\citep{gong2025scaling,ye2025dream,nie2024scaling}, achieving performance on par with AR models. Parallel efforts explored unified multimodal variants that model text and images both in discrete token~\citep{yang2025mmada, li2025lavida}. However, because masked diffusion models do not allow unmasked tokens to change, errors can accumulate during generation due to suboptimal unmasking in earlier steps. Several enhanced (re-)masking techniques have been proposed, using bits and simplex representation to enrich the binary choice of masking~\citep{chao2025mdmprime,song2025shortlisting}, remasking during the reverse process~\citep{gat2024discrete, zhao2024informed, wang2025remasking}, enabling edit operations~\citep{havasi2025edit,song2025seed}.  

\paragraph{Continuous Relaxations for Discrete Data}
Early continuous approaches either learn denoising in a latent embedding without explicit statistical structure~\citep{li2022diffusion, dieleman2022continuous,chen2022analog,zhang2023planner,gulrajani2023likelihoodbased} or fully relax tokens into unconstrained Euclidean space as simplex~\citep{han2022ssd, karimi-mahabadi-etal-2024-tess, tae-etal-2025-tess2}. However, such unconstrained relaxations often fail to preserve the inherent discreteness and categorical semantics of language~\citep{gulrajani2023likelihoodbased}. More recent methods impose structure in the logit space~\citep{hoogeboom2021multinomial, graves2023bayesian} or directly on the probability simplex via Dirichlet priors~\citep{avdeyev2023dirichlet, stark2024dirichlet}, enforcing stronger statistical constraints on the noising process. Flow-matching techniques further treat the simplex as a statistical manifold~\citep{liu2023mirror,cheng2024categorical, davis2024fisherflow}, yet these approaches still lag behind discrete diffusion models in generation fidelity. Recently, \citet{zhang2025flexible} leveraging density models with normalizing flow~\citep{zhai2025normalizing,gu2025starflow} for flexible language modeling, and \citet{sahoo2025duo} connect discrete diffusion language models and the underlying Gaussian diffusion.

\paragraph{Bridging Through the Lens of Mode Balancing}
Our work is also motivated by balancing mode seeking and mode covering. Related efforts pursue this balance via guidance methods that tune the diversity–precision trade-off~\citep{dhariwal2021image,ho2022classifier}; score-distillation approaches that sharpen samples while retaining diffusion training for coverage~\citep{poole2022dreamfusion,song2023consistency,luo2023diffinstruct,yin2024one,zhou2024score,zhang2025target}; and techniques that improve GAN mode coverage using diffusion or augmentation~\citep{zheng2021exploiting,zheng2023truncated,wang2023diffusiongan,karras2020training,zhao2020diffaugment}. Similar effects have been observed when distilling in a paired continuous space~\citep{sahoo2025duo}. From this perspective, we consider the discrete path in CADD possesses the mode-seeking behavior, as the unmasked tokens anchor the modes in the embedding space. The augmented continuous space spreads probability mass to cover plausible alternatives for the next token to enhance the mode coverage.

\section{Preliminary}
Let $\vx_0=(\vx_0^1,\dots,\vx_0^n)$ represent a sequence of discrete tokens from vocabulary $\mathcal V = \{1, 2, ..., V\} \cup \{\vm\}$, which contains $V$ tokens plus a mask token $\vm$ (\texttt{[MASK]}). Each position $\vx_0^i$ is a one-hot vector in $\{0,1\}^{V+1}$.
Let $\vw_\theta:\mathcal V\to\mathbb R^d$ be a learnable token embedding matrix. The embedding representations are obtained as $\vz_0:=\vw_\theta(\vx_0)$, where $\vz_0 \in \R^{n\times d}$. 
\paragraph{Discrete Diffusion Models}
The forward diffusion process is performed through an element-wise conditional sampler $q(\vx_{t}|\vx_0)=\prod_{i=1}^n q(\vx_{t}^i|\vx^i_0)$, defined as ($\delta(\cdot)$ denotes the dirac function):
\begin{equation} \label{eq:background:diffusion_kernel_t0}
q(\vx_t^i|\vx_0^i) \triangleq \alpha_t \delta(\vx_t^i -\vx_0^i) + (1-\alpha_t) \delta(\vx_t^i - \vm),
\end{equation}
where $\alpha_t \in [0,1]$ is a strictly decreasing scheduling function following $\alpha_t=\prod_{s=1}^t(1-\beta_s)$. The reverse process aims to learn $p(\vx_{s}|\vx_t)$ for $0 \le s < t \le 1$. This is typically achieved by training a model $p_\theta(\vx_0|\vx_t)$ to predict the original data from a corrupted state, optimized by minimizing a variational bound on the negative log-likelihood, denoting $\alpha^\prime_t$ the derivative of $\alpha_t$ \textit{w.r.t.} $t$:
\begin{equation} \label{eq:background:diffusion_elbo}
\mathcal{L}_\text{vb}(\boldsymbol{x}_0;\theta) \triangleq \mathbb{E}_{t, \vx_t \sim q(\cdot|\vx_0)} \left[ -\frac{\alpha'_t}{1-\alpha_t} \log p_\theta(\mathbf{x}_0|\mathbf{x}_t) \right].
\end{equation}

\paragraph{Continuous Diffusion Models}
Continuous diffusion models corrupt real-valued data $\vz_0 \!\!\in\!\! \mathbb{R}^{n \times d}$ by adding Gaussian noise scheduled by $\{ \gamma_t\}_{t=1}^T$. The forward process $q(\vz_t|\vz_0)$ is a Gaussian distribution with a closed form:
\begin{equation} \label{eq:background:continuous_forward}
q(\vz_t | \vz_0) = \mathcal{N}(\vz_t; \sqrt{\bar{\gamma}_t}\vz_0, (1-\bar{\gamma}_t)\mI)
\end{equation}
where $\bar{\gamma}_t$ is analogous to $\alpha_t$, with $\bar\gamma_t=\prod_{s=1}^t\gamma_s$ holding. The reverse process $p_\theta(\vz_{t-1}|\vz_t)$ is trained by fitting a network $f_\theta(\cdot)$ with a MSE objective reweighted by signal-to-noise ratio (SNR) function $\lambda(\bar \gamma_t, t)$:
\begin{equation} \label{eq:background:diffusion_elbo_continuous}
\mathcal{L}_\text{vb}(\vz_0;\theta) \triangleq \mathbb{E}_{t, \mathbf{x}_t \sim q(\cdot|\vz_0)} \left[ \lambda(\bar \gamma_t, t) \| f_\theta(\vz_t; t) - \vz_0 \|^2  \right].
\end{equation}

\section{Continuously Augmented Discrete Diffusion (CADD)}
Here we introduce Continuously Augmented Discrete Diffusion (CADD). The high-level intuition is to mitigate the sudden information loss that occurs when tokens are replaced by an absorbing state in discrete diffusion. Inspired by the smooth signal degradation in Gaussian diffusion, CADD augments the discrete state space with a continuous latent variable, $\vz_t$. This variable is paired with discrete tokens $\vx_t$ and is designed to retain semantics of a token's original signal even when tokens in $\vx_t$ are masked. Guided by a  set of latent vectors $\{\vz_t^{(k)}\}_{k=1}^K$, the model predicts next tokens by:
\begin{equation}
    p_\theta(\vx_{t-1} \mid \vx_t) = \E_{\vz_t} [p_\theta(\vx_{t-1} \mid \vx_t, \vz_t)] \approx \frac{1}{K}\sum_{k=1}^K p_\theta(\vx_{t-1} \mid \vx_t, \vz_t^{(k)}).
\end{equation}
Conditioning continuous view of the underlying content at step $t$ and traverse on the $\vz_t$ space, the expectation averages over plausible continuous states so the predictor could realize the distribution of the possible tokens more accurately. Noted that although we may use continuous-time notation $s$ and $t$ for diffusion steps, to improve readability, here we denote specific consecutive steps in the diffusion process by $t$ and $t-1$, with total $T$ steps. Below we present the construction of CADD with main derivations. For more detailed ELBO derivations and proofs, please refer to Appendix~\ref{sec:appendix-elbo}.

\subsection{Forward}
To let $\vz_t$ retain semantic hints of tokens in $\vx_t$ when they are masked, we define the joint transition, which can be factorized as the transitions between discrete tokens, as well as those in the paired continuous space:
\begin{equation}
q(\vx_t,\vz_t\mid \vx_{t-1},\vz_{t-1},\vx_0)
\;:=\;
\underbrace{q(\vx_t\mid \vx_{t-1})}_{\text{discrete part}}\;\cdot\;\underbrace{q(\vz_t\mid \vz_{t-1}, \vx_{t-1},\vx_t, \vx_0)}_\text{continuous part},
\label{eq:def-onestep}
\end{equation}
Given a fixed discrete schedule $\{\beta_t\}_{t=1}^T\in[0,1)^T$ and continuous diffusion schedule $\{\gamma_t\}_{t=1}^T$, the forward transition of discrete and continuous part can be written as following with $\bar\gamma_t:=\prod_{s=1}^t\gamma_s$:
\begin{align}
    &q(\vx_t\mid \vx_{t-1})= \prod_{i=1}^n \Cat \big(\vx_t^i;\ \mQ_t^\top \vx_{t-1}^i\big), ~~~ \mQ_t=(1-\beta_t)\mI+\beta_t\,\mathbf 1\,\bm m^\top.
\label{eq:def-discrete}\\
&q(\vz_t\mid \vz_{t-1}, \vx_{t-1},\vx_t, \vx_0)
=\prod_{i=1}^n
\begin{cases}
\delta(\vz_t^i-\vz_{t-1}^i), & \vx_t^i\neq\vm,\\
\mathcal N\!\big(\vz_t^i;\ \sqrt{\bar \gamma_t}\,\vz_{t-1}^i,\ (1-\bar\gamma_t)\mI_d\big), & \vx_t^i=\vm, \vx_{t-1}^i\neq\vm,\\
\mathcal N\!\big(\vz_t^i;\ \sqrt{\gamma_t}\,\vz_{t-1}^i,\ (1-\gamma_t)\mI_d\big), & \vx_t^i=\vm, \vx_{t-1}^i=\vm.
\end{cases}
\label{eq:def-continuous}
\end{align}
The discrete path still follows the Markov chain and the status at $t$ only depends on the last time step. The continuous path is thus affected by the status at the last time $t-1$ in the latent pace, as well as how the discrete token changes between these two steps, \textit{e.g.}, whether this token is masked for the first time or it is already masked/unmasked.

\begin{wrapfigure}{r}{0.6\textwidth} 
    \centering
    \includegraphics[width=0.6\textwidth]{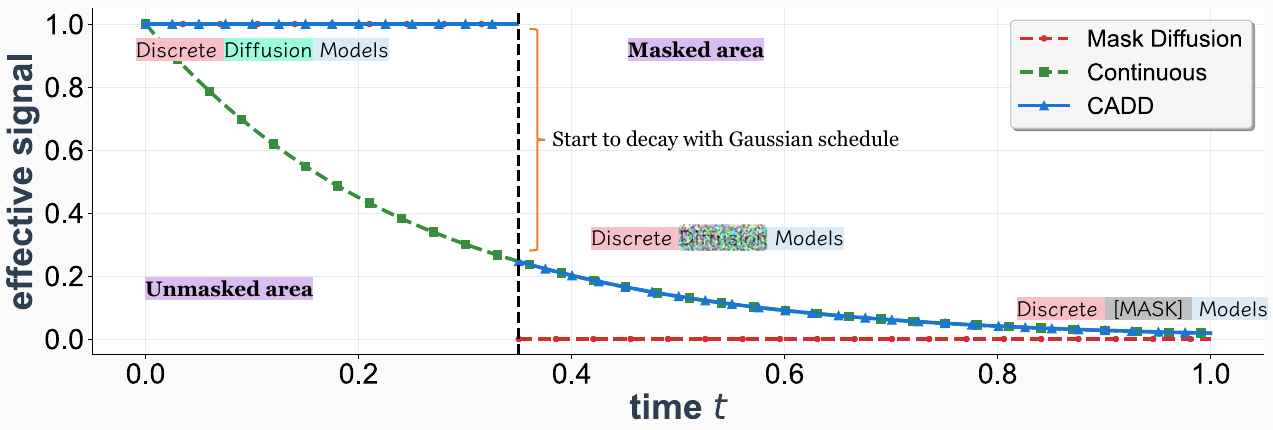} \vspace{-4mm}
    \caption{Example of Signal-to-Noise Ratio (SNR) change of \emph{one token} in the forward of vanilla Mask Diffusion vs. CADD (\textit{Best view in color}). After the second token is masked at the first time, CADD gradually corrupt the information of this token with Gaussian diffusion in the latents, resulting in a smooth decay.}
    \label{fig:snr_example}
\end{wrapfigure}
As a result, the discrete transition is the same as normal discrete diffusion like \citet{austin2021d3pm} and acts as a trigger for the continuous embedding's evolution. 
The continuous trajectory for an embedding remains dormant as long as its token is unmasked, holding its value constant at its original state ($\delta(\vz_t^i - \vz_{t-1}^i) = \delta(\vz_t^i - \vz_{0}^i)$ if $\vx_t^i$ is never masked as the information is not changed). The moment a token is masked, it triggers the continuous diffusion process for its embedding. The embedding then begins a smooth degradation path determined by the Gaussian diffusion~\citep{ho2020ddpm}. If a token stays masked, its embedding simply continues along this path, becoming progressively noisier. \Figref{fig:snr_example} illustrates how our forward process differs from vanilla Mask Diffusion. When all tokens are visible, the SNR for both Mask Diffusion and CADD equals 1. Once a token is masked, the SNR in Mask Diffusion drops to 0 because the absorbing \texttt{[MASK]} carries no token-specific signal. In CADD, the paired continuous latent at that position follows a Gaussian diffusion, so its SNR decays smoothly over time, reflecting graded corruption rather than an abrupt loss.

Now we extend the case to the marginals at timestep $t$ with the following proposition.

\begin{proposition}[Timestep-$t$ joint marginal factorization]
\label{prop:factorization-vector}
The marginal at timestep $t$ can be factorized:
\begin{equation}
{\quad q(\vx_t,\vz_t\mid \vx_0)\ =\ q(\vx_t\mid \vx_0)\ \cdot\ q(\vz_t\mid \vx_t,\vx_0)\quad}
\label{eq:qt-factorizes-vector}
\end{equation}
Given $\alpha_t:=\prod_{s=1}^t(1-\beta_s)$ and $\overline \mQ_t:=\prod_{s=1}^t \mQ_s=\alpha_t \mI+(1-\alpha_t)\,\mathbf 1\,\bm m^\top$ and $\bar\gamma_t:=\prod_{s=1}^t\gamma_s$, with $\vz_0^i=\vw_\theta(\vx_0^i)$, the two terms factorized above represent the discrete and continuous part:
\begin{align}
    &q(\vx_t\mid \vx_0)=\prod_{i=1}^n q(\vx_t^i\mid \vx_0^i), \quad q(\vx_t^i\mid \vx_0^i)=\Cat(\vx_t^i;\ \overline \mQ_t^\top \vx_0^i). \label{eq:def-discrete-marginal}\\
&q(\vz_t\mid \vx_t,\vx_0)=\prod_{i=1}^n q(\vz_t^i\mid \vx_t^i,\vx_0^i)=\prod_{i=1}^n
\begin{cases}
\delta(\vz_t^i-\vz_0^i), & \vx_t^i=\vx_0^i,\\
\mathcal N\!\big(\vz_t^i;\ \sqrt{\bar\gamma_t}\,\vz_0^i,\ (1-\bar\gamma_t)\mI_d\big), & \vx_t^i=\vm,
\end{cases} \label{eq:continuous_marginal_forward}
\end{align}

\end{proposition}
A key property of the marginal distribution $q(\vx_t, \vz_t \mid \vx_0)$ is that it conveniently factorizes into discrete and continuous components: $q(\vx_t \mid \vx_0)$ and $q(\vz_t \mid \vx_t, \vx_0)$. This factorization is highly advantageous, as the distribution for each component is tractable and can be computed in closed form according to the predefined diffusion schedule.

\begin{figure}[t]
    \centering
    \includegraphics[width=\linewidth]{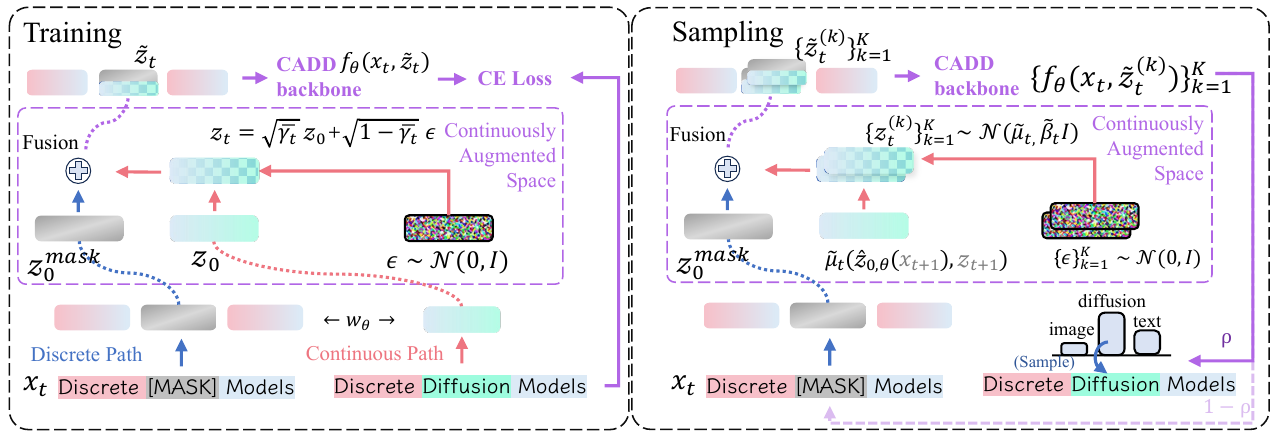}\vspace{-2mm}
    \caption{(\emph{Best view in color}) Illustrative depiction of CADD model, combining both the discrete and continuous feature of the data. In training, the clean token at the masked position will be created by embedding matrix and used to form the noisy embedding according to the continuous forward. In sampling, the model is able to predict a diverse distribution of possible tokens by sampling multiple $\vz_{t}$. Then the predicted tokens will be recycled into the embedding matrix to form $\hat \vz_{0, \theta}$ for the next iteration.  }
    \label{fig:model-illustration}
\end{figure}

\subsection{Reverse}
Following \citet{kingma2021variational,xiao2021tackling,zhou2023beta}, we choose the conditional distribution parameterized with neural network $f_\theta(\cdot)$ to define:
\begin{align}
p_\theta(\vx_{t-1},\vz_{t-1}\mid \vx_t,\vz_t)
&:= q(\vx_{t-1},\vz_{t-1}\mid \vx_t,\vz_t, \vx_0 = \hat{\vx}_0),
\label{eq:reverse-mixture-def}\\
p_\theta(\hat{\vx}_0\mid \vx_t,\vz_t)&=\mathrm{Categorical}\big(\mathrm{logits}=f_\theta(\vx_t,\vz_t)\big) \text{  if  } \vx_t=\vm \text{  else } \delta(\hat \vx_0 - \vx_t).
\end{align}
The objective is to close the gap between the defined parametric distribution and the true posterior. Below we presenet the close form of the posterior. For notation simplicity, below we discuss on per position formulation and omit the notation $i$, since all distributions factorize across positions $i\in\{1,\dots,n\}$. 

\begin{proposition}[Factorization of the true posterior] \label{prop:posterior-factorization}
By the forward construction, the posterior can be factorized in the following form
\begin{equation}
q(\vx_{t-1},\vz_{t-1}\mid \vx_t,\vz_t,\vx_0)
= \underbrace{q(\vx_{t-1}\mid \vx_t,\vx_0)}_\text{discrete part}\;\cdot\;\underbrace{q(\vz_{t-1}\mid \vx_t,\vz_t,\vx_{t-1},\vx_0)}_\text{continuous part}.
\label{eq:true-post-factor}
\end{equation}
Moreover, we can write the close form of each component:
\begin{equation}
q(\vx_{t-1}|\vx_t,\vx_0) = \frac{q(\vx_t|\vx_{t-1})q(\vx_{t-1}|\vx_0)}{q(\vx_t|\vx_0)} \\
= \begin{cases}
\frac{\alpha_{t-1} - \alpha_t}{1-\alpha_t} \vx_{t-1}^\top \vx_0 & \vx_{t-1} \neq \vm, \vx_t = \vm \\
\frac{1-\alpha_{t-1}}{1-\alpha_t} & \vx_{t-1} = \vm, \vx_t = \vm \\
\vx_{t-1}^\top \vx_t & \vx_t \neq \vm.
\end{cases}
\label{eq:disc-post-weights} 
\end{equation}
\begin{equation}
q(\vz_{t-1}\mid \vx_t,\vz_t,\vx_{t-1},\vx_0)=
\begin{cases}
\delta(\vz_{t-1}-\vz_0),& \vx_t=\vx_0 \;\;(\text{no mask at }t),\\
\delta(\vz_{t-1}-\vz_0),& \vx_t=\vm,\ \vx_{t-1}=\vx_0 \;\;(\text{first unmask at }t),\\
\mathcal N\!\big(\vz_{t-1};\,\tilde\vmu_t,\ \tilde\beta_t \mI_d\big),& \vx_t=\vm,\ \vx_{t-1}=\vm,
\end{cases}
\label{eq:cont-post} 
\end{equation}
with the following paramters:
\begin{equation}
\tilde\beta_t=\frac{(1-\bar\gamma_{t-1})\,(1-\gamma_t)}{1-\bar\gamma_t},
\qquad
\tilde\vmu_t=\frac{\sqrt{\bar\gamma_{t-1}}\,(1-\gamma_t)}{1-\bar\gamma_t}\,\vz_0
+\frac{\sqrt{\gamma_t}\,(1-\bar\gamma_{t-1})}{1-\bar\gamma_t}\,\vz_t.
\label{eq:vp-posterior-params}
\end{equation}
\end{proposition}

\begin{lemma}\label{lemma:obj}
    For the unmasked positions ($\vx_t\neq\vm$), the KL is identically $0$, and the masked positions splits exactly as
\begin{equation}
\label{eq:exact-split}
\KL\!\big(q(\cdot\mid \vx_t,\vz_t,\vx_0)\,\big\|\,p_\theta(\cdot\mid \vx_t,\vz_t)\big)
=\underbrace{\rho_t^{\mathrm{flip}}\ \big[-\log p_{\theta}(\vx_0|\vx_t, \vz_t)\big]}_{\text{discrete}}
+\underbrace{ \rho_t^{\mathrm{keep}}\ \mathcal \KL^{\mathrm{cont}}}_{\text{continuous}},
\end{equation} 
with the ratio that determines whether the position is going to be flipped to unmask or keep moving in the continuous space:
\begin{equation}
\rho_t^{\mathrm{keep}}=\frac{1-\alpha_{t-1}}{1-\alpha_t}, 
\qquad
\rho_t^{\mathrm{flip}}=\frac{\alpha_{t-1}\,\beta_t}{1-\alpha_t} = \frac{\alpha_{t-1}-\alpha_t}{1-\alpha_t}. 
\label{eq:rho_compute}
\end{equation}
The KL divergence in the continuous space has a reweighted MSE form:
\begin{equation}
\label{eq:cont-single-gauss-exact}
\KL^{\mathrm{cont}}
=\frac{1}{2\tilde\beta_t}\,\big\|\tilde\vmu_t(\vz_0,\vz_t)-\tilde\vmu_t(\hat \vz_{0,\theta},\vz_t^i)\big\|^2
=\frac{a_t^2}{2\tilde\beta_t}\;\|\vz_0-\hat \vz_{0,\theta}\|^2;\ a_t=\frac{\sqrt{\bar\gamma_{t-1}}(1-\gamma_t)}{1-\bar\gamma_t}.
\end{equation}
\end{lemma}

\begin{figure}[t]
\begin{minipage}[t]{0.48\textwidth}
\begin{algorithm}[H] 
\caption{Training of CADD}\label{alg:cadd-train}
\begin{algorithmic}[1]
\small
\State \textbf{Input:} dataset $\mathcal{X}$, network $f_\theta(\cdot)$, masking schedule $\{\alpha_t\}_{t=1}^T$, continuous schedule $\{\bar \gamma_t\}_{t=1}^T$
\While{not converged}
    \State draw data $\vx_0 \sim \mathcal{X}$, draw $t \sim \mathrm{Uniform}({1,...,T})$
    \State mask out each token position $\vx_0^{i}$ with probability $1 - \alpha_{t}$ to obtain $\vx_{t}$
    \State form discrete embeddings $\vz_\mathrm{disc} \!\!\gets\!\! \vw_\theta(\vx_{t})$, 
    \State form continuous embeddings \If{ $\vx_{t}^{i} = \vm$} $\vz_t \!\!\gets\!\! \vw_\theta(\vx_0)$  \Else ~~$\vz_t \!\!\gets \vzero$ \EndIf
    \For{position $i \in \{1, ..., n\}$, if $\vx_{t}^{i} = \vm$}, 
    $\vz_{t}^{i} \!\!\gets\!\! \sqrt{\bar \gamma_{t}} \vz_{t}^{i} + \sqrt{1-\bar \gamma_{t}}\veps,  \veps\sim (\vzero, \mI) $ 
    \EndFor
    \State $\tilde \vz_{t} \gets \vz_\mathrm{disc} + \vz_t$, compute logits $f_\theta(\tilde \vz_{t})$
    \State optimize with cross entropy loss in \eqref{eq:training loss}
\EndWhile
\normalsize
\end{algorithmic}
\end{algorithm}
\end{minipage}% 
\hfill
\begin{minipage}[t]{0.48\textwidth}
\begin{algorithm}[H]
\caption{Sampling of CADD}\label{alg:cadd-sample}
\begin{algorithmic}[1]
\small
\State \textbf{Input:} desired number of samples $B$, network $f_\theta(\cdot)$, schedules $\{\alpha_t\}_{t=1}^T$, $\{\bar \gamma_t\}_{t=1}^T$, 
\While{not reach desired size $B$ }
\State \textbf{init: } \!\!$\vx_T \!\gets \! (\vm, ... \vm)$, $\vz_T \!\overset{\text{i.i.d.}}{\sim}\! \mathcal{N}(\vzero, \mI)$
\For{$t = T, \dots, 1$}
\For{$i = 1, \dots, n$, if $\vx_{t}^{i} = \vm$   }
\State \!\!compute $\rho_{t}^{\!\!\mathrm{flip}\!\!}$ and $\rho_t^{\!\!\mathrm{keep}\!\!}$ (\eqref{eq:rho_compute})
\State determine whether to unmask $\vx_{t-1}^{i}  \sim \mathrm{Cat}(\rho_{t}^{\mathrm{flip}}  f_\theta(\vx_{t}^{i}, \vz_{t}^{i})   +   \rho_t^{\mathrm{keep}}\vm)$
\If{$\vx_{t-1}^{i} \gets \vm$ }  
draw $\vz_{t-1}^i\sim \mathcal N\!\Big(\tilde \vmu_t\big(\hat \vz_{0,\theta}^i,\ \vz_t^i\big),\ \tilde\beta_t \mI_d\Big)$ with \eqref{eq:vp-posterior-params}
\Else 
~~$\vz_{t-1}^{i} \gets \vw_\theta(\vx_{t-1}^{i})$ 
\EndIf 
\EndFor
\EndFor
\EndWhile
\normalsize
\end{algorithmic}
\end{algorithm}
\end{minipage}

\end{figure}\vspace{-2pt}

\subsection{Algorithm and Implementation}
Given the results above, below comes the training and sampling algorithms. The model design is illustrated in~\Figref{fig:model-illustration} regarding how the model is trained and how it handles one sampling step.
\paragraph{Training Loss}
According to \eqref{eq:exact-split}, the model aims to learn to maximize the likelihood of discrete path, and also minimize the reweighted MSE in \eqref{eq:cont-single-gauss-exact}.  Inspired by continuous diffusion models that used for categorical modeling, e.g., CDCD~\citep{dieleman2022continuous} and Plaid~\citep{gulrajani2023likelihoodbased}, we may estimate  $\hat{\vz}_{0,\theta}:=\sum_v p_\theta(\hat \vx_0 = v \mid \vx_t,\vz_t)\,\vw_{\theta, v}$ and just train the model to predict correct categorical output to minimize the KL divergence. Thus, we choose to train CADD by minimizing a simple cross entropy loss as following and the training is summarized in Algorithm~\ref{alg:cadd-train}:
\begin{equation}
    \boxed{\gL_\mathrm{CADD} =\E_{t \sim \mathrm{Uniform}({1,...,T})}\E_{q(\vx_t, \vz_t \mid \vx_0)}\big[-\sum_{i:\,\vx_t^i=\vm} \log p_{\theta}(\vx_0^i \mid \vx_t^i, \vz_t^i)\big]}
\label{eq:training loss}
\end{equation}
Note that we may add the MSE loss in \eqref{eq:cont-single-gauss-exact} to the above objective to more accurately estimate the exact variational lower bound. Empirically we find the simplified loss is more computationally efficient, thus we choose to use this loss for most of our experiments unless otherwise specified. 

\paragraph{Sampling}
The sampling start from the last timestep $T$ of the diffusion chain. Under the absorbing forward, $\alpha_T\approx 0$, hence $p(\vx_T)=\delta_{\vx_T=\vm}$, i.e., all tokens are masked.
Since all positions are masked at $T$, the continuous prior is
$p(\vz_T|\vx_T)=\prod_{i=1}^n \mathcal N\big(\vz_T^i;\ \vzero,\ \mI_d\big)$, 
which matches the forward marginal at $T$. For each timestep, given $(\vx_t,\vz_t)$, the network
predicts
$$
\pi_{\theta,i}(v)\ := \frac{1}{K}\sum_{k=1}^K\ p_\theta(\hat \vx_0^i=v\mid \vx_t,\vz_t^{(k)})\in\Delta^{V-1}\qquad\text{for each position }i.
$$
For an unmasked position, the absorbing chain keeps $\vx_{t-1}^i=\vx_t^i$ almost surely and
the continuous variable is deterministic $\vz_{t-1}^i = \vz_t^i = \vw_\theta(\vx_t^i)$. For a masked position, with probability $\frac{1-\alpha_{t-1}}{1-\alpha_t}$, it draws a clean token $v\sim \pi_{\theta,i}(\cdot)$  to unmask it. If this masked position is unmasked in this step, the continuous latent $\vz_{t-1}^i \gets w_{\theta,v}$. If it remains masked, $\vz_{t-1}^i$ moves along the continuous diffusion trajectory $\vz_{t-1}^i\sim \mathcal N\!\Big(\tilde \vmu_t\big(\hat \vz_{0,\theta}^i,\ \vz_t^i\big),\ \tilde\beta_t \mI_d\Big)$ following \eqref{eq:vp-posterior-params}.
The full sampling  process is shown in Algorithm~\ref{alg:cadd-sample}. Note the choice of $\hat \vz_{0,\theta}^i$ has two options:

\begin{equation}
    \textbf{hard: } \ \hat \vz_{0,\theta} = \vw_\theta(\hat \vx_0), \hat \vx_0 = \argmax_v \pi_{\theta,i}(v) \quad \textbf{soft: }  \hat{\vz}_{0,\theta}:=\sum_v \pi_{\theta,i}(v)\,\vw_{\theta, v}.
\label{eq:zhat_choice}
\end{equation}
These two choices are both valid to use depending on whether we are looking for mode-covering or mode-seeking  behavior, i.e., better context localization or better diversity, respectively. In our main experiments we keep the hard option, and our empirical exploration in Appendix~\ref{sec:appendix-ablation} justify these two choices could meet the demand of these two behavior. Moreover, although CADD may leverage multi-sample for the $\vx_0$ distribution estimation, for fair comparison with baselines, we keep $K=1$ for most of our experiments. More detailed studies are also shown in the Appendix~\ref{sec:appendix-ablation}.

\paragraph{Implementation}
We follow the common-used design of the model architecture to let $f_\theta(\cdot)$ predict logits for categorical distribution. The discrete path follows earlier masked-diffusion setups: starting from $\vx_0$, we mask a
subset of positions to obtain $\vx_t$, embed the mixed sequence with the learnable table and form $\vz_\mathrm{disc} = \vw_\theta(\vx_t)$. The only difference is the model needs to take an additional variable $\vz_t$ input for the continuous embeddings. To achieve this, we first form the clean embeddings $\vz_0 = \vw_\theta(\vx_0)$, and then apply noise only at masked positions using the forward marginal \eqref{eq:continuous_marginal_forward} to obtain $\vz_t$. We fuse $\vz_\mathrm{disc}$ and $\vz_t$ by element-wise addition $\tilde{\vz}_t := \vz_{\mathrm{disc}} + \vz_t$,
and feed $\tilde{\vz}_t$ to the backbone $f_\theta$ to produce per-position logits.

\section{Experiments}
In this section we present experiments to validate the proposed CADD model through experiments on text, image, and code generation benchmarks. The evaluations are designed to assess the model's performance across diverse data modalities and scales.

\subsection{Text Generation}

\paragraph{Experiment setting} 
For text generation, we strictly follow the experimental setup of the Masked Diffusion Language Model (MDLM)~\citep{sahoo2024simple}, a common configuration for this task. We train our CADD models on the OpenWebText (OWT) dataset~\citep{Gokaslan2019OpenWeb}. Data is tokenized using the GPT-2 tokenizer with a vocabulary size of $ |\mathcal{V}| =50,257$~\citep{radford2019gpt2}, and sequences are fixed to a length of $n\!=\!1,024$. To be consistent with the baselines, we use a Discrete DiT backbone~\citep{peebles2023dit} with approximately 168M parameters, and  train with same number of iterations. All training hyper-parameters are identical to those in MDLM.

\paragraph{Evaluation.} 
We mainly compare the performance with discrete diffusion baselines in terms of the generative quality, and our evaluation protocol strictly follows that of~\citet{wang2025remasking}. We compare the performance against discrete diffusion baselines using two metrics: the MAUVE score (higher is better)~\citep{liu2021divergence, pillutla2021mauve} and generative perplexity (lower is better)~\citep{lou2024sedd}. Further details on the evaluation setup are located in Appendix~\ref{sec:appendix-exp-settinngs}.

\begin{figure}[t]
  \centering
  \begin{subfigure}[t]{0.495\textwidth}
    \centering
    \includegraphics[width=\linewidth]{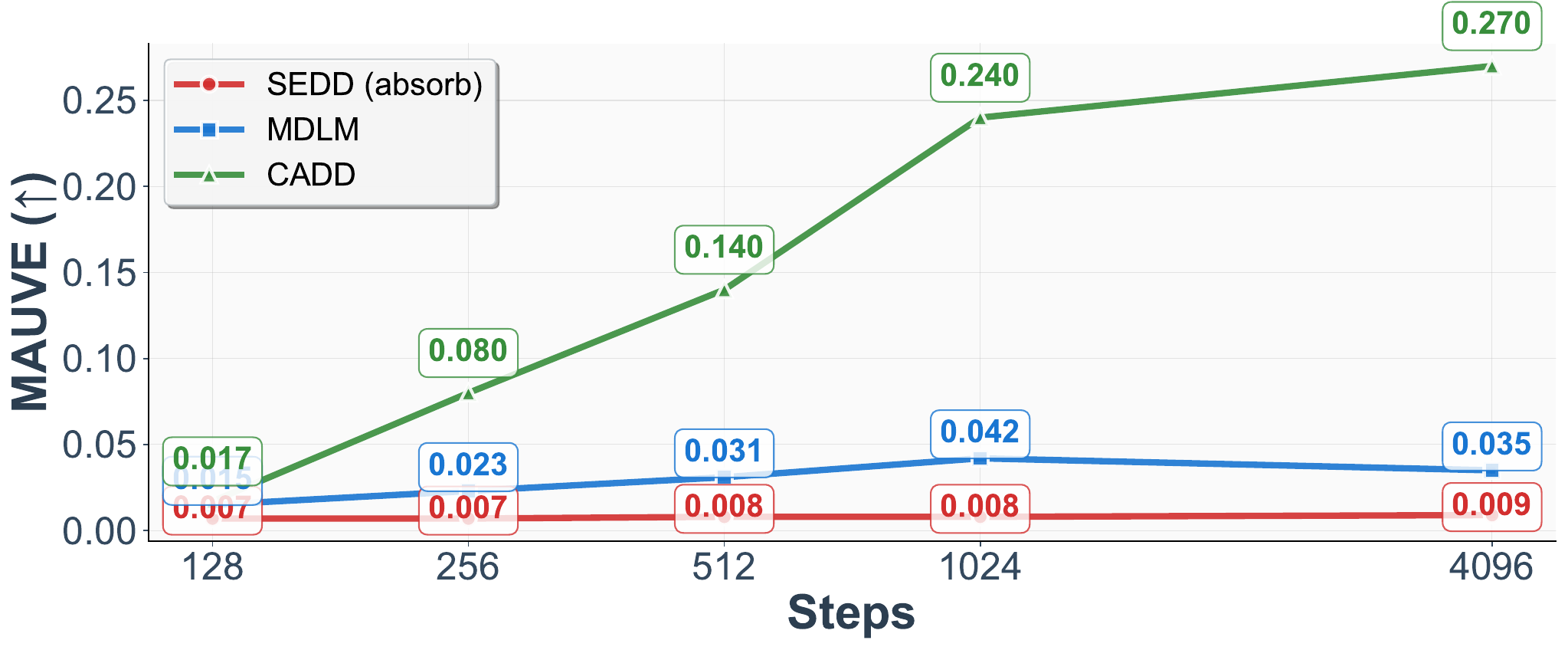}\vspace{-5pt}
  \end{subfigure}\hfill
  \begin{subfigure}[t]{0.495\textwidth}
    \centering
    \includegraphics[width=\linewidth]{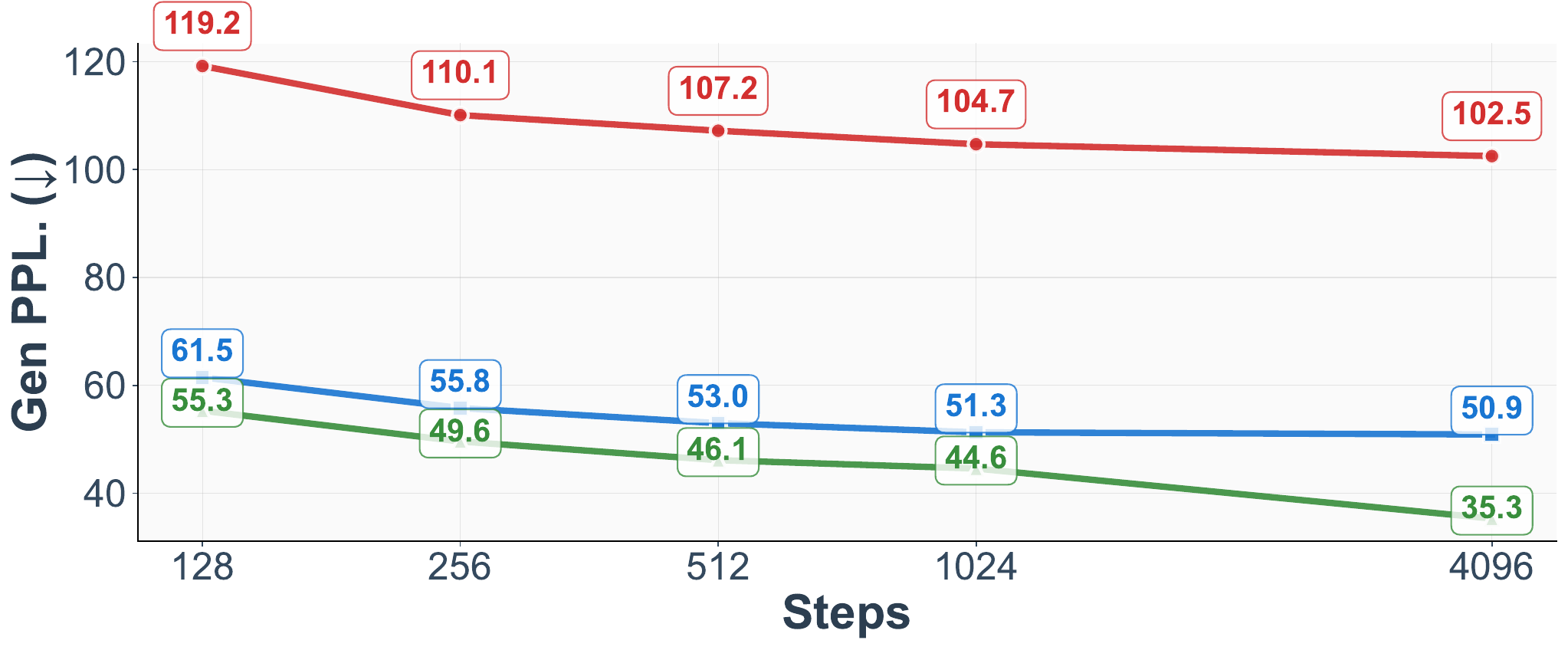}\vspace{-5pt}
  \end{subfigure}\vspace{-2mm}
  \caption{Unconditional text generative evaluation of model trained on OpenWebText (OWT) data. All method are evaluated with 128, 256, 512 1024, and 4096 sampling steps. MAUVE (\textit{Left Panel}, higher is better) and generative perplexity (\textit{Right Panel}, measured using GPT2-Large, lower is better) are reported.}
  \label{fig:owt-main}
\end{figure}

\paragraph{Main Results.} \Figref{fig:owt-main} presents the results for unconditional text generation on the OpenWebText (OWT) dataset, comparing CADD with SEDD (absorb) and MDLM across a range of sampling steps $T\in \{$128, 256, 512, 1024, 4096$\}$.  Within the range $T \leq 1024$, all models show improvement as the number of sampling steps increases. We can notice CADD demonstrates stronger and consistent gains as steps increase compared to SEDD and MDLM in terms of both metrics. Plotting the x-axis on a $\log_2$ scale reveals that the performance trend is approximately linear. 

Extending the sampling process to $T = 4096$ further demonstrates CADD's scaling capabilities at inference time, as it continues to improve while the masked-only baselines stagnate or degrade. From $T=1024$ to $4096$, CADD’s MAUVE score still increases by 0.3, and its generative perplexity is scored from 44.6 to 35.3. MDLM's performance slightly worsens, which is consistent with the observation that mask-only diffusion models scale poorly with $T$~\citep{wang2025remasking}. Overall, CADD consistently show performance gain across all tested number of sampling steps over the mask-only discrete diffusion models, validating the effectiveness of the proposed continuous-augmented space.

\paragraph{Computation}
With our design, the number of trainable parameters in the network is actually the same as MDMs, which is 168M for the used DiT architecture. We also measure the inference time for 5000 samples on 8 H100 GPUs, where both MDLM and CADD take 0.5h. When the number of samples used for $\hat \vz_0$ is 1, i.e., $K=1$, the computation in the network is comparable since we only have extra computation in the forward and the fusion (add) operation. The computation cost increases linearly as $K$ goes greater than 1.

\begin{table}[t]
\centering
\small
\begin{minipage}[t]{0.49\textwidth}

\caption{FID and IS evaluation on CIFAR-10. The arrow symbols denote lower/higher is better respectively. Baseline results are quoted from \citet{chao2025mdmprime}. }
\label{tab:cifar10}
\setlength{\tabcolsep}{1.0mm}{ 
\begin{tabular}{lcc}
\toprule[1.5pt]
\textbf{Method} & FID ($\downarrow$) & IS ($\uparrow$) \\
\midrule
\rowcolor{gray!30} CADD (NFE=512)   & \textbf{2.88} & \textbf{10.04} \\
\midrule
\textcolor{gray}{\textbf{Discrete}} & & \\
MDM (NFE=512)           & 4.66 & 9.09 \\
MDM-Mixture (NFE=512)   & 4.80 & 9.22 \\
{MDM-Prime (NFE=512)} & {3.26} & {9.67} \\
D3PM Absorb (NFE=1,000) & 30.97 & 6.78 \\
D3PM Gauss. (NFE=1,000) & 7.34 & 8.56 \\
CTDD-DG (NFE=1,000) & 7.86 & 8.91 \\
Tau-LDR (NFE=1,000) & 3.74 & 9.49 \\
Discrete FM (NFE=1,024) & 3.63 & - \\
\midrule
\textcolor{gray}{\textbf{Continuous}} & & \\
Continuous FM & 6.35 & - \\
Bit Diffusion & 3.48 & - \\
StyleGAN+ADA & 3.26 & {9.74} \\
DDPM& {3.17} & 9.46 \\
\bottomrule[1.5pt]
\end{tabular}
}
\end{minipage}
\hfill
\begin{minipage}[t]{0.48\textwidth}
\centering
\caption{FID evaluation using model unconditionally trained on ImageNet ($32\times32$ resolution).}
\label{tab:imagenet32}
\setlength{\tabcolsep}{1.0mm}{ 
\begin{tabular}{lc}

\toprule[1.5pt]
\textbf{Method} & FID ($\downarrow$) \\
\midrule
\rowcolor{gray!30} CADD (NFE=1,024)   & \textbf{3.74} \\
\midrule
\textcolor{gray}{\textbf{Discrete}} & \\
MDM (NFE=1,024)           & 7.91 \\
MDM-Mixture (NFE=1,024)   & 8.08 \\
{MDM-Prime (NFE=1,024)} & {6.98}  \\
\midrule
\textcolor{gray}{\textbf{Continuous}} & \\
NDM & 17.02 \\
DDPM & 16.18  \\
MSGAN & 12.30  \\
i-DODE (SP) & 10.31 \\
i-DODE (VP) & 9.09  \\
Stochastic Interp. & 8.49  \\
Soft Trunc. DDPM & 8.42 \\
ScoreFlow (subVP) & 8.87  \\
ScoreFlow (VP) & 8.34  \\
Continuous FM & {5.02}  \\
\bottomrule[1.5pt]
\end{tabular}
}
\end{minipage}
\vspace{-2mm}
\end{table}

\subsection{Image Generation}
We train and evaluate our models on the CIFAR-10~\citep{cifar10} and ImageNet~\citep{krizhevsky2017imagenet} datasets (resolution $32\times 32$). For both, input images are in RGB channels, thus a dimensionality of $n=32\times32\times3$ with $|\mathcal{V}|=256$ pixel values per channel. For fair comparison the MDM baselines, our model architecture follows the one used in \citet{chao2025mdmprime,gat2024discrete}, which is based on the ADM~\citep{dhariwal2021image} architecture. We choose MDM-Prime~\citep{chao2025mdmprime} and its variants as our main discrete diffusion baseline. We also include its discrete and continuous diffusion model baselines for comparison~\citep{shih2022training, ho2020ddpm, song2021scorebased, austin2021d3pm, campbell2022ctmc, gat2024discrete, nisonoff2025unlocking,lipman2022flow, chen2022analog, bartosh2023neural, tran2019msgan, zheng2023ode, albergo2023stochastic, kim2022soft}. To assess sample quality, we report Fréchet Inception Distance (FID) and Inception Score (IS), computed with 50,000 random samples. 

We follow MDM variants to unconditionally sample images with same number of function evaluation (NFE) and report results on CIFAR-10 in Table~\ref{tab:cifar10}. With the same NFE, we can observe CADD improves upon MDMs by a significant margin. Attaining an FID of 2.88 and an Inception Score of 10.04 with 512 function evaluations (NFE), CADD surpasses the MDM variants by 0.38 in terms of FID and represents the best result among all compared method.  On ImageNet-32, as shown in Table~\ref{tab:imagenet32}, the observation is constent, where CADD obtains FID of 3.74 and outperforms all reported baselines. The qualitative generated samples are provided in Appendix~\ref{sec:appendix-gen-results} for visual justifications.

\subsection{Code Generation}

For a large-scale setting, we conduct code generation experiments based on the DiffuCoder pipeline~\citep{gong2025diffucoder}. The DiffuCoder base model training process involves adapting a pretrained autoregressive LLM (e.g., Qwen2.5-coder~\citep{hui2024qwen2}) into a discrete diffusion model by annealing its attention mechanism from causal to bidirectional~\citep{gong2025scaling}. The resulting model is then trained using a masking diffusion loss~\citep{shi2024md4}. In this context, we evaluate our method using the following two distinct configurations. (i) Vanilla CADD: We follow the DiffuCoder procedure to adapt the Qwen2.5-coder model. Instead of using the MDM loss, we train the model from the beginning with our proposed CADD loss. (ii) CADD (fine-tuned): To demonstrate CADD's effectiveness as a fine-tuning objective, we initialize our model from a pretrained DiffuCoder checkpoint and then continue training it with the CADD loss. To ensure a fair comparison, both CADD variants are trained on the same 65B total tokens and use the same training hyperparameters as the original DiffuCoder. In the evaluation, we follow their settings to test the model performance on three coding benchmarks: HumanEval~\citep{chen2021evaluating}, MBPP~\citep{austin2021program}, and BigCodeBench~\citep{zhuo2024bigcodebench}. 

\begin{table}[t]
    \centering
    \small
    \setlength{\tabcolsep}{4pt} 
    \caption{Benchmark coding capacities of AR and Diffusion LLMs in 7/8B scale. We follow the evaluation settings in DiffuCoder~\citep{gong2025diffucoder}, where EvalPlus is computed as the average of HE+ and MBPP+. The best performance in AR and Diffusion LLMs are marked in bold. }%\vspace{-2mm}
    \setlength{\tabcolsep}{1.0mm}{ 
    \begin{tabular}{ l ll ll c cc c }
        \toprule[1.5pt]
        \multirow{2}{*}{\textbf{Model}}
            & \multicolumn{2}{c}{HumanEval}
            & \multicolumn{2}{c}{MBPP}
            & \multirow{2}{*}{EvalPlus}
            & \multicolumn{2}{c}{BigCodeBench (C)} 
            & \multirow{2}{*}{Avg.}\\
        \cmidrule(lr){2-3} \cmidrule(lr){4-5} \cmidrule(lr){7-8}
            & -   & Plus
            & -   & Plus
            & 
            & \multicolumn{1}{c}{Full} & \multicolumn{1}{c}{Hard} & \\
        \midrule
        \textcolor{gray}{\textbf{AR}} & \multicolumn{8}{l}{} \\
        Qwen2.5-Coder      & 61.6  & 51.8  & 75.9  & 61.4  & 56.6  & \textbf{46.1}  & 16.2  &52.2 \\
        OpenCoder~\citep{huang2024opencoder}      & 66.5  & \textbf{63.4}  & \textbf{79.9}  & \textbf{70.4}  & 66.9  & 40.5  & 9.5  & \textbf{55.0} \\
        \arrayrulecolor{gray}\midrule
        \textcolor{gray}{\textbf{Diffusion}} & \multicolumn{8}{l}{} \\
        LLaDA~\citep{nie2025llmdiff}         & 35.4  & 30.5  & 50.1  & 42.1  & 36.3  & 18.9  & 4.1  & 30.2 \\
        Dream~\citep{ye2025dream}         & 56.7  & 50.0  & 68.7  & 57.4  & 53.7  & 23.6  & 4.1  & 43.4\\
        DiffuCoder            & {67.1}  & {60.4}  & 74.2  & 60.9  & {60.7}  & 40.2  & {12.8}  & {52.6} \\
        \midrule
        \rowcolor{gray!30} CADD (ours)            & {72.0}  & {63.4}  & \textbf{75.7}  & \textbf{63.2}  & \textbf{63.3}  & \textbf{42.1}  & \textbf{17.6}  & \textbf{55.7}\\
        \rowcolor{gray!30} CADD (ours, DiffuCoder init)            & \textbf{73.8}  & \textbf{64.6}  & {73.9}  & {60.4}  & {62.5}  & {41.5}  & {15.5}  & {55.0}\\
        \bottomrule[1.5pt]
    \end{tabular}}
    \label{tab:code-gen}%\vspace{-4mm}
\end{table}

Table~\ref{tab:code-gen} reports the pass@1 performance, where the results of both autoregressive (AR) and diffusion-based LLMs are included, with an overall average score provided. Compared with Diffusion-based models, CADD emerges as the strongest diffusion model, outperforming competitors on nearly all metrics. Compared to the previous leading DM, DiffuCoder, CADD significantly improves performance on HumanEval, e.g., from 67.1 to 72.0; on the challenging BigCodeBench-Hard subset, we can also observe significant performance gain from 12.8 to 17.6. CADD is also highly competitive with leading AR code models. It surpasses Qwen2.5-Coder across all benchmarks and achieves a higher overall average than OpenCoder (55.7 vs. 55.0). When using Diffucoder's checkpoint as initialization for continuous space finetuning, we also find CADD improves the Diffucoder's performance on HumanEval (73.8 vs. 67.1) and BigCodeBench (41.5 vs. 40.2).

\section{Conclusion}
In standard discrete diffusion, information is lost abruptly when tokens are replaced by an absorbing state. Inspired by Gaussian diffusion, where the data signal degrades smoothly, CADD's core idea is to introduce an auxiliary continuous space to guide the discrete process. This space is designed to retain semantic information, providing a smooth continuous representation of a token even after its discrete form has been absorbed. By conditioning on it, the model can better be aware of what was supposed to be in the masked position. This leads to more coherent and contextually accurate generations, as the model has a stronger grasp of the underlying meaning. With extensive empirical justification on text, image and code generation, we justify that with the continuous augmented space proposed in CADD, the discrete diffusion models consistently generate higher quality samples  across these different tasks and achieve strong performance.

\section*{Acknowledgment}
The authors thank Josh Susskind, Irina Belousova, Miguel Angel Bautista, Richard Bai, Shuangfei Zhai, Tatiana Likhomanenko, Xiaoming Zhao, Yuyang Wang and Zijin Gu for insightful feedbacks and discussions. We also thank Marco Cuturi Cameto and Miguel Angel Bautista for helping setup the template of our arXiv version.

\bibliography{reference}
\bibliographystyle{arxiv}

\clearpage

%%%%%%%%%%%%%%%%%%%%%%%%%%%%%%%%%%%%%%%%%%%%%%%%%%%%%%%%%%%%

\appendix

\section{Detailed Derivations and Proof}\label{sec:appendix-elbo}
\subsection{ELBO Derivation}
\textbf{Forward chain.}
For any observation $\vx_0$, the forward diffusion constructs as
\begin{equation}
q(\vx_{1:T},\vz_{1:T} \mid \vx_0)
= \prod_{t=1}^{T} q_t\!\big(\vx_t,\vz_t \mid \vx_{t-1},\vz_{t-1},\vx_0\big),
\label{eq:q-forward}
\end{equation}
note we represent $\big(\vx_0,\vz_0\big)$ as $\vx_0$ since the transform $\vw_\theta$ is deterministic. 

\textbf{Reverse generative model.}
\begin{equation}
p_\theta(\vx_0, \vx_{1:T}, \vz_{1:T})
= p_T(\vx_T,\vz_T)\;
\Big[\prod_{t=2}^{T} p_\theta\!\big(\vx_{t-1},\vz_{t-1}\mid \vx_t,\vz_t\big)\Big]\;
p_\theta(\vx_0\mid \vx_1,\vz_1).
\label{eq:p-reverse}
\end{equation}

%Assume absolute continuity so that all Kullback--Leibler divergences are finite and all expectations below exist.

\begin{proposition}[ELBO decomposition]
\label{prop:elbo}
Given the forward chain $q$ defined in \eqref{eq:q-forward} and reverse model $p_\theta$ in \eqref{eq:p-reverse}, we have the decomposed ELBO as following: 
\begin{align}
\log p_\theta(\vx_0)
&\ge
\underbrace{\E_{q(\vx_1,\vz_1\mid \vx_0)}\!\big[\log p_\theta(\vx_0\mid \vx_1,\vz_1)\big]}_{\text{reconstruction term at }t=1}\nonumber\\
&\quad
-\underbrace{\sum_{t=2}^{T}\E_{q(\vx_t,\vz_t\mid \vx_0)}
\Big[ \KL\!\big(q(\vx_{t-1},\vz_{t-1}\mid \vx_t,\vz_t,\vx_0)\,\big\|\,p_\theta(\vx_{t-1},\vz_{t-1}\mid \vx_t,\vz_t)\big)\Big]}_{\text{denoising matches for }t>1}\nonumber\\
&\quad
-\underbrace{\KL\!\big(q(\vx_T,\vz_T\mid \vx_0)\,\|\,p_T(\vx_T,\vz_T)\big)}_{\text{prior match at }T}.
\label{eq:main-elbo}
\end{align}
If $q(\vx_T,\vz_T\mid \vx_0)=p_T(\vx_T,\vz_T)$ for all $\vx_0$, then the prior match term is zero. The bound is tight if and only if
\[
p_\theta(\vx_{t-1},\vz_{t-1}\mid \vx_t,\vz_t)
= q(\vx_{t-1},\vz_{t-1}\mid \vx_t,\vz_t,\vx_0)\quad\text{for all }t\ge 2,
\]
and the prior match is zero, and the decoder $p_\theta(\vx_0\mid \vx_1,\vz_1)$ equals the true conditional induced by the joint.
\end{proposition}

Recap the forward kernel defined in \eqref{eq:def-discrete} and \eqref{eq:def-continuous}:

\begin{equation}
q(\vx_t\mid \vx_{t-1})= \prod_{i=1}^n \Cat \big(\vx_t^i;\ \mQ_t^\top \vx_{t-1}^i\big), ~~~ \mQ_t=(1-\beta_t)\mI+\beta_t\,\mathbf 1\,\bm m^\top.
\notag
\end{equation}
\begin{equation}
q(\vz_t\mid \vz_{t-1}, \vx_{t-1},\vx_t, \vx_0)
=\prod_{i=1}^n
\begin{cases}
\delta(\vz_t^i-\vz_{t-1}^i), & \vx_t^i\neq\vm,\\
\mathcal N\!\big(\vz_t^i;\ \sqrt{\bar \gamma_t}\,\vz_{t-1}^i,\ (1-\bar\gamma_t)\mI_d\big), & \vx_t^i=\vm, \vx_{t-1}^i\neq\vm,\\
\mathcal N\!\big(\vz_t^i;\ \sqrt{\gamma_t}\,\vz_{t-1}^i,\ (1-\gamma_t)\mI_d\big), & \vx_t^i=\vm, \vx_{t-1}^i=\vm.
\end{cases}
\notag
\end{equation}

\begin{proof}[Proof of Proposition~\ref{prop:elbo}]
The proof is mostly done in \citet{dickstein2015diffusion} and \citet{ho2020ddpm}. We include the following proof to show the generalized version with added variables. Start from the evidence identity and apply Jensen inequality:
\begin{align}
\log p_\theta(\vx_0)
&=\log \int q(\vx_{1:T},\vz_{1:T}\mid \vx_0)\,
\frac{p_\theta(\vx_0,\vx_{1:T},\vz_{1:T})}{q(\vx_{1:T},\vz_{1:T}\mid \vx_0)}\;d\vx_{1:T}\,d\vz_{1:T}\nonumber\\
&\ge
\E_{q(\vx_{1:T},\vz_{1:T}\mid \vx_0)}\!\Big[
\log p_\theta(\vx_0,\vx_{1:T},\vz_{1:T})-\log q(\vx_{1:T},\vz_{1:T}\mid \vx_0)
\Big]\nonumber\\
&=: \gL(\theta;\vx_0).
\label{eq:elbo-jensen}
\end{align}
Insert the model and forward factorizations \eqref{eq:p-reverse} and \eqref{eq:q-forward}:
\begin{align}
\gL(\theta;\vx_0)
&=\E_q\Big[
\log p_T(\vx_T,\vz_T)
+\sum_{t=2}^T \log p_\theta(\vx_{t-1},\vz_{t-1}\mid \vx_t,\vz_t) \\
&+\log p_\theta(\vx_0\mid \vx_1,\vz_1)
-\sum_{t=1}^T \log q(\vx_t,\vz_t\mid \vx_{t-1},\vz_{t-1},\vx_0)\Big].
\label{eq:elbo-expanded}
\end{align}
For each $t\ge 2$ use Bayes’ rule under $q$:
\begin{equation}
q(\vx_t,\vz_t\mid \vx_{t-1},\vz_{t-1},\vx_0)
=\frac{q(\vx_{t-1},\vz_{t-1}\mid \vx_t,\vz_t,\vx_0)\;q(\vx_t,\vz_t\mid \vx_0)}{q(\vx_{t-1},\vz_{t-1}\mid \vx_0)}.
\label{eq:bayes-q}
\end{equation}
Taking $\E_q[\log(\cdot)]$ of \eqref{eq:bayes-q} and rearranging gives, for $t\ge 2$,
\begin{align}
&\E_q\!\Big[\log p_\theta(\vx_{t-1},\vz_{t-1}\mid \vx_t,\vz_t)
-\log q(\vx_t,\vz_t\mid \vx_{t-1},\vz_{t-1},\vx_0)\Big]\nonumber\\
&\qquad=
-\E_{q(\vx_t,\vz_t\mid \vx_0)}\!\Big[
\KL\!\big(q(\vx_{t-1},\vz_{t-1}\mid \vx_t,\vz_t,\vx_0)\ \|\ p_\theta(\vx_{t-1},\vz_{t-1}\mid \vx_t,\vz_t)\big)
\Big]\nonumber\\
&\qquad\quad
-\E_q\!\big[\log q(\vx_t,\vz_t\mid \vx_0)\big]
+\E_q\!\big[\log q(\vx_{t-1},\vz_{t-1}\mid \vx_0)\big].
\label{eq:one-step-id}
\end{align}
Sum \eqref{eq:one-step-id} over $t=2,\dots,T$. The last two expectations telescope:
\begin{align}
-\sum_{t=2}^T\E_q\!\big[\log q(\vx_t,\vz_t\mid \vx_0)\big]
&+\sum_{t=2}^T\E_q\!\big[\log q(\vx_{t-1},\vz_{t-1}\mid \vx_0)\big]\nonumber\\
= \E_q\!\big[\log q(\vx_1,\vz_1\mid \vx_0)\big]-&\E_q\!\big[\log q(\vx_T,\vz_T\mid \vx_0)\big].
\end{align}
Plug this back into \eqref{eq:elbo-expanded} and group the boundary terms with $\log p_T$:
\begin{align}
\gL(\theta;\vx_0)
&=\E_q\big[\log p_\theta(\vx_0\mid \vx_1,\vz_1)\big]\nonumber\\
&\quad-\sum_{t=2}^T \E_{q(\vx_t,\vz_t\mid \vx_0)}\!\Big[
\KL\!\big(q(\vx_{t-1},\vz_{t-1}\mid \vx_t,\vz_t,\vx_0)\ \|\ p_\theta(\vx_{t-1},\vz_{t-1}\mid \vx_t,\vz_t)\big)
\Big]\nonumber\\
&\quad-\Big(\E_q\!\big[\log q(\vx_T,\vz_T\mid \vx_0)\big]-\E_q\!\big[\log p_T(\vx_T,\vz_T)\big]\Big)\nonumber\\
&\quad-\E_q\!\big[\log q(\vx_1,\vz_1\mid \vx_0)\big].
\label{eq:elbo-with-const}
\end{align}
Now we recoginize the prior KL to obtain
\begin{align}
\gL(\theta;\vx_0)
&=\E_q\big[\log p_\theta(\vx_0\mid \vx_1,\vz_1)\big]\nonumber\\
&\quad-\sum_{t=2}^T \E_{q(\vx_t,\vz_t\mid \vx_0)}\!\Big[
\KL\!\big(q(\vx_{t-1},\vz_{t-1}\mid \vx_t,\vz_t,\vx_0)\ \|\ p_\theta(\vx_{t-1},\vz_{t-1}\mid \vx_t,\vz_t)\big)\Big]\nonumber\\
&\quad-\KL\!\big(q(\vx_T,\vz_T\mid \vx_0)\ \|\ p_T(\vx_T,\vz_T)\big)\;-\;\underbrace{\E_q\!\big[\log q(\vx_1,\vz_1\mid \vx_0)\big]}_{=:C(\vx_0)}.
\label{eq:elbo-final-const}
\end{align}
Note the last term $C(\vx_0)$ does not involve $p_\theta$ and can be dropped, and we normally do not optimize the last KL term $\KL\!\big(q(\vx_T,\vz_T\mid \vx_0)\ \|\ p_T(\vx_T,\vz_T)\big)$ as we let the schedule to make this statistical distance is sufficiently small.
\end{proof}

\subsection{Forward}
We can derive the following lemma for the marginal at time step $t$.

\begin{lemma}[Continuous marginal conditioned on $(\vx_t,\vx_0)$]
\label{lem:zt-marginal-vector}
Recap $\bar\gamma_t:=\prod_{s=1}^t\gamma_s$. For each position $i$, we have continuous marginal conditioned on $(\vx_t,\vx_0)$ as
\begin{equation}
q(\vz_t^i\mid \vx_t^i,\vx_0^i)=
\begin{cases}
\delta(\vz_t^i-\vz_0^i), & \vx_t^i=\vx_0^i,\\
\mathcal N\!\big(\vz_t^i;\ \sqrt{\bar\gamma_t}\,\vz_0^i,\ (1-\bar\gamma_t)\mI_d\big), & \vx_t^i=\vm,
\end{cases}
\notag
\end{equation}
with $\vz_0^i=\vw_\theta(\vx_0^i)$. Hence We finally have
\begin{equation}
q(\vz_t\mid \vx_t,\vx_0)
=\prod_{i=1}^n q(\vz_t^i\mid \vx_t^i,\vx_0^i)
=\Big[\!\!\!\prod_{i:\,\vx_t^i\neq\vm}\!\!\!\delta(\vz_t^i-\vz_0^i)\Big]
\cdot
\Big[\!\!\!\prod_{i:\,\vx_t^i=\vm}\!\!\!\mathcal N(\vz_t^i;\sqrt{\bar\gamma_t}\vz_0^i,(1-\bar\gamma_t)\mI_d)\Big].
\notag
\end{equation}
\end{lemma}

Then what follows proves Proposition~\ref{prop:posterior-factorization}. We first prove the conditional independency between $\vz_t$ and $\vx_{t-1}$ given $(\vx_t, \vx_0)$ in the reverse context.
\begin{lemma}[Conditional independency between $\vz_t$ and $\vx_{t-1}$ given $(\vx_t, \vx_0)$]\label{lemma:cond-indp}
    $\vz_t$ and $\vx_{t-1}$ are conditionally independent given $(\vx_t, \vx_0)$ based on the forward kerned defined in \eqref{eq:def-continuous}. 
\end{lemma}

To prove Proposition~\ref{prop:factorization-vector}, we first prove the following lemma:

\begin{proof}[Proof of Lemma~\ref{lem:zt-marginal-vector} and Lemma~\ref{lemma:cond-indp}]
If $\vx_t^i=\vx_0^i$ then the absorbing chain implies $\vx_s^i\neq \vm$ for $s\le t$, so the kernel gives
$\vz_t^i=\vz_0^i$ almost surely, which is the first line of \eqref{eq:def-continuous}.

If $\vx_t^i=\vm$, use the law of total probability over $\vx_{t-1}^i\in\{\vx_0^i,\vm\}$.

When $\vx_{t-1}^i\neq \vm$ (first time masking at $t$), the second branch of the kernel gives
$\vz_t^i\sim \mathcal N(\sqrt{\bar\gamma_t}\,\vz_0^i,(1-\bar\gamma_t)\mI)$.

When $\vx_{t-1}^i=\vm$ (already masked), the third branch composes a diffusion forward step with the
previous marginal $\vz_{t-1}^i\sim \mathcal N(\sqrt{\bar\gamma_{t-1}}\,\vz_0^i,(1-\bar\gamma_{t-1})\mI)$,
which yields
\[
\vz_t^i\ \sim\ \mathcal N\!\big(\sqrt{\gamma_t\bar\gamma_{t-1}}\,\vz_0^i,\ (1-\gamma_t\bar\gamma_{t-1})\mI\big)
=\mathcal N\!\big(\sqrt{\bar\gamma_t}\,\vz_0^i,\ (1-\bar\gamma_t)\mI\big).
\]
This proves the masked line of \eqref{eq:def-continuous}.
\end{proof}

Then leveraging these results, we can easily prove Proposition~\ref{prop:factorization-vector}.

\begin{proof}[Proof of Proposition~\ref{prop:factorization-vector}]
Expand the path marginal, use \eqref{eq:def-onestep} and Lemma~\ref{lem:zt-marginal-vector}, and factor over positions.
The sum over discrete paths yields $q(\vx_t\mid \vx_0)$; conditioning on $\vx_t$ reduces the continuous part to Lemma~\ref{lem:zt-marginal-vector}.
\end{proof}

\subsection{Reverse}

\begin{proof}[Proof of Proposition~\ref{prop:posterior-factorization}]
We first prove the factorization shown in \eqref{eq:true-post-factor}. To achieve this, we just need to show:
\begin{align}
q(\vx_{t-1},\vz_{t-1}\mid \vx_t,\vz_t,\vx_0) & = q(\vx_{t-1}\mid \vx_t,\vz_t,\vx_0) \cdot q(\vz_{t-1}\mid \vx_t,\vz_t,\vx_{t-1},\vx_0)  \\
& = \frac{ q(\vz_{t}\mid \vx_{t-1},\vx_t,\vx_0) q(\vx_{t-1}\mid \vx_t,\vx_0)}{ q(\vz_{t}\mid \vx_t, \vx_0)} \cdot q(\vz_{t-1}\mid \vx_t,\vz_t,\vx_{t-1},\vx_0)
\\
& = q(\vx_{t-1}\mid \vx_t,\vx_0) \cdot q(\vz_{t-1}\mid \vx_t,\vz_t,\vx_{t-1},\vx_0),
\end{align}
where $q(\vz_{t}\mid \vx_{t-1},\vx_t,\vx_0) = q(\vz_{t}\mid \vx_t, \vx_0)$ by the conditional independence according to Lemma~\ref{lemma:cond-indp}. Then the discrete part is the same as discrete diffusion, we may leverage the results from \citet{austin2021d3pm,sahoo2024simple,shi2024md4} to complete the proof of \eqref{eq:disc-post-weights}. 

Next, we prove the closed form of the continuous part, $q(\vz_{t-1}\mid \vx_t,\mathbf{z}_t,\vx_{t-1},\vx_0)$, by case analysis based on the discrete states. We start with Bayes' rule for the continuous variables:
\begin{equation}
q(\vz_{t-1} \mid \vx_t,\vz_t,\vx_{t-1},\vx_0) \propto q(\vz_t \mid \vz_{t-1}, \vx_t) \cdot q(\vz_{t-1} \mid \vx_{t-1}, \vx_0).
\label{eq:bayes_posterior_z}
\end{equation}
The forms of the two terms on the right-hand side are Gaussian distributions, but will change depending on the discrete states and it leads to the three cases. 

\textbf{Case 1: No mask at $t$ ($\vx_t = \vx_0$).} In this case, no noise has been applied to the embedding up to timestep t-1. Thus, both terms directly have a Dirac delta function: $q(\vz_{t-1} \mid \vx_{t-1}=\vx_0, \vx_0) = \delta(\vz_{t-1}-\vz_0)$. The posterior is therefore also a Dirac delta function, proving the first part of \eqref{eq:cont-post}.

\textbf{Case 2: First time unmask at $t$ ($\vx_t = \vm$, $\vx_{t-1}=\vx_0$).} In this case, the first term in \eqref{eq:bayes_posterior_z} is Gaussian while the second term becomes a Dirac $\delta(\vz_{t-1} - \vz_0)$. The multiplication yields a Dirac posterior at the same point: $q(\vz_{t-1} \mid \vx_{t-1}=\vx_0, \vx_0) = \delta(\vz_{t-1}-\vz_0)$. 

\textbf{Case 3: Remaining masked at $t$ ($\vx_t = \vm$, $\vx_{t-1}=\vm$).} In this case, both terms remain in Gaussian distribution, and the parameters are same with normal Gaussian diffusion models. The product of these two Gaussians is a new Gaussian, allowing us to use the standard derivation for DDPM \citep{ho2020ddpm}, by completing the square on the exponent, we find that the resulting distribution is $\mathcal{N}(\mathbf{z}_{t-1};\,\tilde{\boldsymbol{\vmu}}_t(\mathbf{z}_t, \mathbf{z}_0),\ \tilde{\beta}_t\mathbf{I})$, which proves the last part of \eqref{eq:cont-post}. 

\end{proof}

\begin{proof}[Proof of Lemma~\ref{lemma:obj}]
Using the results from Proposition~\ref{prop:posterior-factorization}, for a single position $i$, the exact one–step KL at timestep $t > 1$ inside the ELBO is
\begin{equation}
\KL(\vx_0, t)
:=\E_{q(\vx_t,\vz_t\mid \vx_0)}\!\Big[
\KL\!\big(q(\vx_{t-1},\vz_{t-1}\mid \vx_t,\vz_t, \vx_0)\,\big\|\,p_\theta(\vx_{t-1},\vz_{t-1}\mid \vx_t,\vz_t)\big)
\Big],
\label{eq:KL-exact}
\end{equation}
For the unmasked positions ($\vx_t\neq\vm$), the KL is identically $0$, and plug in \eqref{eq:true-post-factor}, \ref{eq:disc-post-weights} and \ref{eq:cont-post}, we recover \eqref{eq:exact-split} exactly as
\begin{equation}
\notag
\KL\!\big(q(\cdot\mid \vx_t,\vz_t,\vx_0)\,\big\|\,p_\theta(\cdot\mid \vx_t,\vz_t)\big)
=\underbrace{\rho_t^{\mathrm{flip}}\ \big[-\log p_{\theta}(\vx_0|\vx_t, \vz_t)\big]}_{\text{discrete}}
+\underbrace{ \rho_t^{\mathrm{keep}}\ \mathcal \KL^{\mathrm{cont}}}_{\text{continuous}},
\end{equation} 
with the ratio that determines whether the position is going to be flipped to unmask:
\begin{equation}
\rho_t^{\mathrm{keep}}=\frac{1-\alpha_{t-1}}{1-\alpha_t},
\qquad
\rho_t^{\mathrm{flip}}=\frac{\alpha_{t-1}\,\beta_t}{1-\alpha_t} = \frac{\alpha_{t-1}-\alpha_t}{1-\alpha_t}.
%,\qquad \pi_{\theta}(v):=p_\theta(\hat \vx_0=v\mid \vx_t,\vz_t).
\notag
\end{equation}
The discrete KL part exactly recovers the results from the absorbing discrete diffusion models \citep{austin2021d3pm,sahoo2024simple,shi2024md4}, and the continuous KL divergence:
\begin{equation}
\label{eq:exact-mix-kl}
\KL^{\mathrm{cont}}
=\KL\!\Big(\mathcal N(\vmu^\star,\tilde\beta_t \mI_d)\ \Big\|\  \mathcal N(\vmu_v,\tilde\beta_t \mI_d)\Big),
\quad
\vmu^\star=\tilde\vmu_t(\vz_0,\vz_t),\ \ \vmu_v=\tilde\vmu_t(\hat \vz_0,\vz_t),
\end{equation}
where we recap
$$
\\
\tilde\vmu_t(\zeta,\vz_t)=\frac{\sqrt{\bar\gamma_{t-1}}(1-\gamma_t)}{1-\bar\gamma_t}\,\zeta+\frac{\sqrt{\gamma_t}(1-\bar\gamma_{t-1})}{1-\bar\gamma_t}\,\vz_t,
\qquad
\tilde\beta_t=\frac{(1-\bar\gamma_{t-1})(1-\gamma_t)}{1-\bar\gamma_t}.
$$
This results in the comparison between $\vz_0$ and $\hat \vz_0$ and the KL divergence reduced to:
\begin{equation}
\notag %label{eq:cont-single-gauss-exact}
\KL^{\mathrm{cont}}
=\frac{1}{2\tilde\beta_t}\,\big\|\tilde\vmu_t(\vz_0,\vz_t)-\tilde\vmu_t(\hat \vz_{0,\theta},\vz_t^i)\big\|^2
=\frac{a_t^2}{2\tilde\beta_t}\;\|\vz_0-\hat \vz_{0,\theta}\|^2;\ a_t=\frac{\sqrt{\bar\gamma_{t-1}}(1-\gamma_t)}{1-\bar\gamma_t}.
\end{equation}
\end{proof}

\begin{remark}[On the Alternative Factorization]
One could also decompose the posterior using the alternative order from the chain rule:
$$ q(\vx_{t-1},\vz_{t-1}\mid \cdot) = q(\vz_{t-1}\mid \vx_t,\vz_t,\vx_0) \cdot q(\vx_{t-1}\mid \vx_t,\vz_t, \vz_{t-1}, \vx_0). $$
While mathematically valid and could provide new properties in the sampling, this factorization is not fully tractable. The first term, $q(\vz_{t-1}|\cdot)$, is a complex Gaussian Mixture Model. More critically, the second term, $q(\vx_{t-1}|\cdot)$, has no analytical closed form, as it would require inverting the continuous diffusion process and the embedding function to infer a discrete state. The factorization in Prop.~\ref{prop:posterior-factorization} is therefore adopted as a tractable choice for a more efficient algorithm implementation.
\end{remark}

\section{Detailed Experiment Settings}\label{sec:appendix-exp-settinngs}
\subsection{Diffusion Settings} 
The CADD forward process has two coupled components, each with its own schedule.
\begin{itemize}
    \item Discrete schedule: we adopt the MDLM log-linear masking schedule for the discrete process~\citep{sahoo2024simple}. The discrete forward corruption uses a continuous-time $\alpha(t)=1-t$,  with $t\in[0,1]$. 
    \item Continuous schedule: to keep the meaning of time aligned, we set the continuous latent $\vz$ to follow a linear flow-matching path to isotropic noise~\citep{lipman2022flow}, i.e., if the position is masked, we have $\vz_t=(1-t)\vz_0 + t\veps$,  $\veps\sim\mathcal N(\vzero,\mI)$. 
    \item Multi-sample estimation: we by default set $K=1$ for the estimation of $\hat \vx_{0, \theta}$ for fair comparison with the baselines. We provide ablation studies to demonstrate the effect of $K>1$. 
\end{itemize}

\subsection{Experiment-Specific Settings}

\textbf{Text Generation}  
In our main experiments, including ablation studies that used to explore the properties of CADD, we train the models on OpenWebText. Following the standard MDLM pre-processing~\citep{sahoo2024simple}, we use the GPT-2 tokenizer, resulting in a vocabulary of 50,257 tokens. The sequence length is fixed at 1,024. Our text model is a 12-layer DiT with 12 attention heads and an embedding dimension of 768, totaling approximately 168M parameters. During training, we keep the same training configuration, i.e., we train for about 2M steps with a batch size of 256 to match the total 262B tokens seen in the training. We use the AdamW optimizer with a learning rate warmed up from 0 to $3 \times 10^{-4}$. The results in Table~\ref{tab:text8} and Table~\ref{tab:lm1b-ppl}, are based on Text8 and LM1B dataset, where we strictly follow the training setting in \citet{jo2024riemannian} and \citet{sahoo2024simple}. Please refer their experiment settings for more details.  For evaluation, we follow ReMDM~\citep{wang2025remasking}'s evaluation setting, where we randomly sample 5,000 text samples with length $n=1,024$, using $\{128, 256, 512, 1024, 4096\}$ sampling steps. The sampled token sequences are used to compute MAUVE score, generative perplexity with GPT2-Large model, and entropy. 

\textbf{Image Generation}  
We experiment on CIFAR-10 and ImageNet (with resolution $32\times32$), which consists of 50,000 and  1,281,149 natural images respectively. CIFAR-10 already has $32\times32$ resolution. For ImageNet images, we follow the preprocessing used in EDM~\citep{Karras2022edm}, i.e., using center-crop to make it a squared image and rescale to the desired $32\times32$ resolution. As the model is trained on pixel space, we treat each pixel as a discrete token, resulting in a vocabulary size 256 at each position. We follow the architecture design used in MDM-Prime~\citep{chao2025mdmprime}, which is a U-Net architecture based on ADM~\citep{dhariwal2021image}. For CIFAR-10, we leverage an augmentation pipeline proposed in \citet{karras2020training}, but only keep the rotation and flip operation to avoid pixel value changes. We set the augmentation probability as 15\% on CIFAR-10, and there is no augmentation used on ImageNet. For both experiments, we set learning rate as $1 \times 10^{-4}$ using AdamW optimizer, and train the model until it has seen 200M and 4B images respectively. In sampling, we adopt a cosine decay for temperature with $\tau_{max}=2.5$, and applied the corrector following~\citet{gat2024discrete}. We use the standard Fréchet Inception Distance (FID) and Inception Score for evaluation, computed with 50,000 randomly generated images.

\textbf{Code Generation}  
We use the OpenCoder dataset~\citep{huang2024opencoder}, selected by following the recipe in DiffuCoder~\citep{gong2025diffucoder}. We strictly follow their settings to initialize the 7B model with Qwen2.5-Coder checkpoint, and adapt it to diffusion model using the techniques introduced in \citet{gong2025scaling}. Then we trained the model on 64 NVIDIA A100 GPUs in total. The training process utilized BF16 mixed precision and was scaled using Fully Sharded Data Parallelism (FSDP). For optimization, we employed the Adam optimizer with a peak learning rate of $1\times 10^{-5}$, preceded by a 2,000-step linear warmup. The model is trained with 65B tokens in total. For generation, both models were configured with a maximum sequence length of 512 tokens and a total of T=512 diffusion timesteps. During generation, we employed a top negative entropy remasking sampler. The CADD from scratch variant uses temperature 0.2 and the DiffuCoder initialized variant uses temperature 0.01.

\section{Additional Experiment Results}
\subsection{Training from mask diffusion model}
From the experiments on code generation, we have seen CADD could be used to finentune an existing discrete (masking) diffusion model to improve the performance. Here we provide complementary evidence that such observation is also valid on text generation. We finetune a MDLM checkpoint with CADD objective for additional 50B tokens and evaluate the performance with same setting shown in the main experiments (\Figref{fig:owt-main}). The results are shown in \Figref{fig:ft-mdlm}. The red curve shows close performance to the green one that represent CADD's performance, which indicates CADD could efficiently finetune an existing MDM model to enhance the generation capabilities.

\begin{figure}[t]
  \centering
  \begin{subfigure}[t]{0.495\textwidth}
    \centering
    \includegraphics[width=\linewidth]{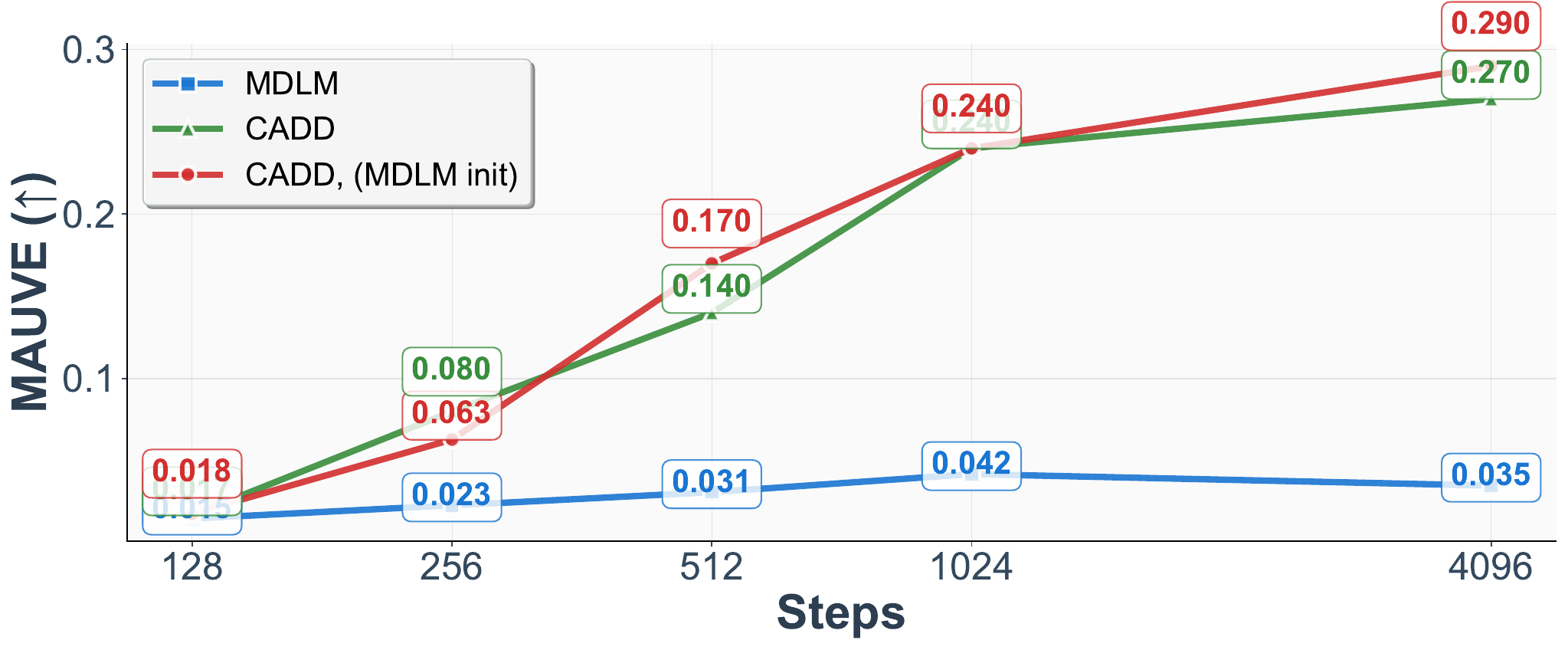}
  \end{subfigure}\hfill
  \begin{subfigure}[t]{0.495\textwidth}
    \centering
    \includegraphics[width=\linewidth]{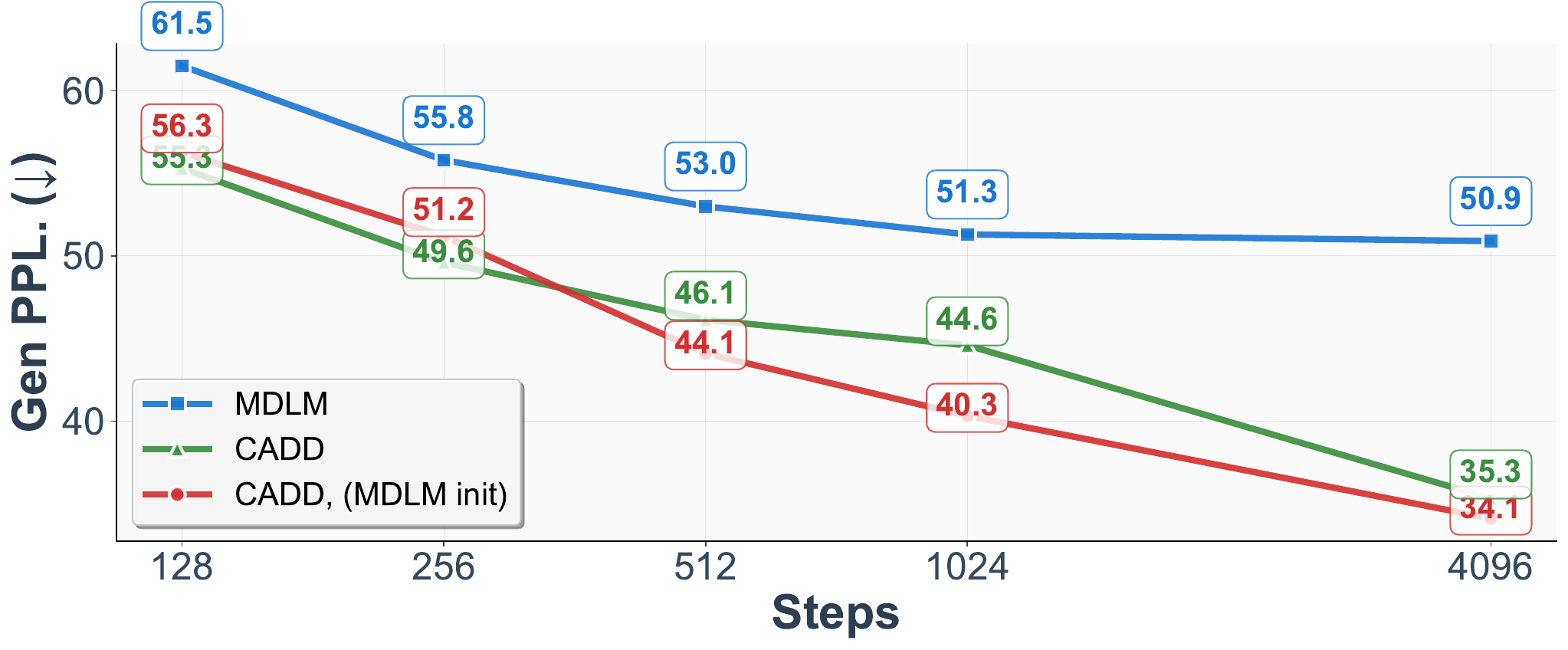}
  \end{subfigure}
  \caption{Analogous figure of~\Figref{fig:owt-main}. We compare the finetuned checkpoint using CADD objective with CADD and the initialization checkpoint of MDLM. }
  \label{fig:ft-mdlm}
\end{figure}

\subsection{Perplexity Evaluation}
Since the objective of CADD involves the KL divergence of both discrete and continuous component as shown in \eqref{eq:true-post-factor}, it is not fair to compare the tightness of the bound directly with other models, and we choose to focus more on the evaluation of the generated samples. However, our model is still able to compute the likelihood of the discrete part. Here we put the results for reference, aiming to provide more information to help the readers understand how the model helps the discrete diffusion side.

\begin{table}[t]
\centering
\begin{minipage}[t]{0.48\textwidth}
\caption{
    {Bits Per Character (BPC)} results on Text8 test set. Results are taken from ~\citet{jo2024riemannian}.
    Bold denotes the best result in autoregressive or diffusion models. The best diffusion results are marked in bold.
}
\label{tab:text8}
\centering
    \resizebox{0.8\columnwidth}{!}{
    \renewcommand{\arraystretch}{0.92}
    \renewcommand{\tabcolsep}{14pt}
\begin{tabular}{l c}
\toprule[1.5pt]
     Method & BPC ($\downarrow$) \\
\midrule
    \textit{Autoregressive} & \\
    AR & \textbf{1.23} \\
\midrule
    \textit{Continuous Diffusion} & \\
    Plaid & $\leq$ 1.48 \\
    BFN & $\leq$ 1.41\\
    RDLM  & $\leq$ \textbf{1.32} \\
\midrule
    \textit{Discrete Diffusion} & \\
    Multinomial Diffusion & $\leq$ 1.72 \\
    D3PM Uniform & $\leq$ 1.61 \\
    D3PM Absorb & $\leq$ 1.45 \\
    SEDD Absorb & $\leq$ 1.39 \\
    MDLM & $\leq$ 1.40 \\
    MD4 & $\leq$ 1.37 \\
    \midrule
     CADD (Ours) & $\leq$ 1.35 \\
\bottomrule[1.5pt]
\end{tabular}}
\end{minipage}
\hfill
\begin{minipage}[t]{0.51\textwidth}
\centering
\caption{Test perplexities (PPL; $\downarrow$) on LM1B. 
  The baseline results are taken from~\citet{sahoo2025duo}.
  For CADD, we report the bound on the discrete likelihood. Best diffusion value is \textbf{bolded}. $^\star$ the dataset for SEDD didn't incorporate sentence packing.
  }
  \label{tab:lm1b-ppl}
  \centering
{
\footnotesize
  
  \begin{tabular}{lcc}
    \toprule[1.5pt]
    Method&  {\scriptsize LM1B} & {\scriptsize OWT} \\
    \midrule
    \multicolumn{2}{l}{\textit{Autoregressive}}\\
         Transformer   & 22.8 & 17.5\\
    \midrule
    \multicolumn{2}{l}{\textit{Diffusion (Uniform-state / Gaussian)}}\\
        D3PM Uniform~{\scriptsize \citep{austin2021d3pm}} &     137.9  &   - \\
          Diffusion-LM$^{*}$~{\scriptsize \citep{li2022diffusion}} &    118.6 &  -\\
         SEDD Uniform~{\scriptsize \citep{lou2024sedd}} &  40.3$^{\star}$ &  29.7 \\
     UDLM~{\scriptsize \citep{deschenaux2025bydiff}} &    36.7 &  27.4\\
     DUO~{\scriptsize \citep{sahoo2025duo}}  &   {33.7} &  {25.2} \\ %
    \midrule
    \multicolumn{2}{l}{\textit{Diffusion (absorbing state)}}\\
        D3PM Absorb~{\scriptsize \citep{austin2021d3pm}} &  76.9 & -\\
        DiffusionBert~{\scriptsize \citep{he2023diffusionbert}} &  63.8& - \\
        SEDD Absorb~{\scriptsize \citep{lou2024sedd}} &  32.7$^{\star}$ & 24.1 \\
        MDLM~{\scriptsize \citep{sahoo2024simple}} &  {31.8}& {23.2}\\
        \midrule
     CADD (Ours) & \textbf{31.4} &  \textbf{23.1} \\
    \bottomrule[1.5pt]
  \end{tabular}}
\end{minipage}
\end{table}

Table~\ref{tab:text8} and Table~\ref{tab:lm1b-ppl} report the perplexity evaluation on character-level and token-level respectively. The model is trained on Text8 and LM1B, following the settings of \citet{jo2024riemannian} and \citet{sahoo2024simple}. On Text8, we can see CADD achieves very competitive perplexity results, and is slightly worse than the SoTA RDLM~\citep{jo2024riemannian}. On LM1B, we can see CADD achieves the best results among diffusion models when evaluating the discrete part perplexity on both LM1B data and OWT data.

Table~\ref{tab:zeroshot-ppl} reports the zero-shot evaluation results of the checkpoint trained on OWT data. We can observe CADD and MDLM both surpasses the perplexity of AR models on Lambada, Pubmed and Arxiv datasets. They have different dataset that they are good at in terms of perplexity, and CADD wins slightly more as it shows better zero-shot perplexity than MDLM on 4/7 tasks. These experiments result jointly indicate that CADD can not only provide strong generation quality, but also provide a good discrete likelihood bound.

\begin{table}[t]
\caption{Zero-shot perplexities (upper bounds) of models trained for 1M steps on OpenWebText. Baseline results are taken from~\citet{sahoo2025duo}. Best diffusion model performance results are \textbf{bolded} and diffusion values better than AR are \underline{underlined}. Plaid and D3PM are trained with 0.3M more steps.
}
\label{tab:zeroshot-ppl}
\centering
{\footnotesize
\begin{tabular}{llccccccc}
\toprule[1.5pt]
Method&& PTB & Wikitext & LM1B & Lambada  & AG News & Pubmed & Arxiv\\
\midrule
\multicolumn{9}{l}{\textit{Autoregressive}} \\
& Transformer &82.05 & 25.75 & {51.25} & 51.28 & {52.09} & 49.01 & 41.73\\
\midrule
\multicolumn{9}{l}{\textit{Diffusion (Uniform-state / Gaussian)}} \\
&  SEDD Unifor &  105.51 &  41.10 &  82.62 &  57.29 &  82.64 &  55.89 &  50.86 \\
&  Plaid &  142.60 &  50.86 &  91.12&  57.28 &  - &  - &  - \\
&  UDLM &  112.82 &  39.42 &  77.59 &  53.57 &  80.96 &  50.98 &  44.08 \\
&  DUO  & {89.35} &  {33.57} &  {73.86} &   \underline{{49.78}} & {67.81} &  \underline{{44.48}} &  \underline{{40.39}}\\
\midrule
\multicolumn{9}{l}{\textit{Diffusion (absorbing state)}} \\
& SEDD Absorb & 100.09 & 34.28 & 68.20& \underline{49.86} & 62.09 & \underline{44.53} & \underline{38.48} \\
& D3PM Absorb & 200.82 & 50.86 & 138.92& 93.47 & - & - & - \\
& MDLM & 95.26 & {32.83} & {67.01} & \underline{47.52} & \textbf{61.15} & \textbf{\underline{41.89}} & \textbf{\underline{37.37}}\\
\midrule
& CADD (Ours)  & \textbf{93.33} & \textbf{31.84} & \textbf{64.98} & \textbf{\underline{46.81}} & {62.80} & \underline{42.62} & \underline{37.52}\\
\bottomrule[1.5pt]
\end{tabular}
}
\end{table}

\subsection{Ablation studies}\label{sec:appendix-ablation}
\paragraph{Comparing the number of samples used for $\hat \vx_0 = f_\theta(\vx_t, \vz_t^{(k)})$}
We first conduct ablation to study how the number of samples used to compute $\hat \vx_0$ would affect CADD's performance. Similar to our main experiments in text generation, we compare CADD with $K\in\{1,2,3,4\}$ in terms of MAUVE and generative perplexity. 

\begin{figure}[t]
  \centering
  \begin{subfigure}[t]{0.495\textwidth}
    \centering
    \includegraphics[width=\linewidth]{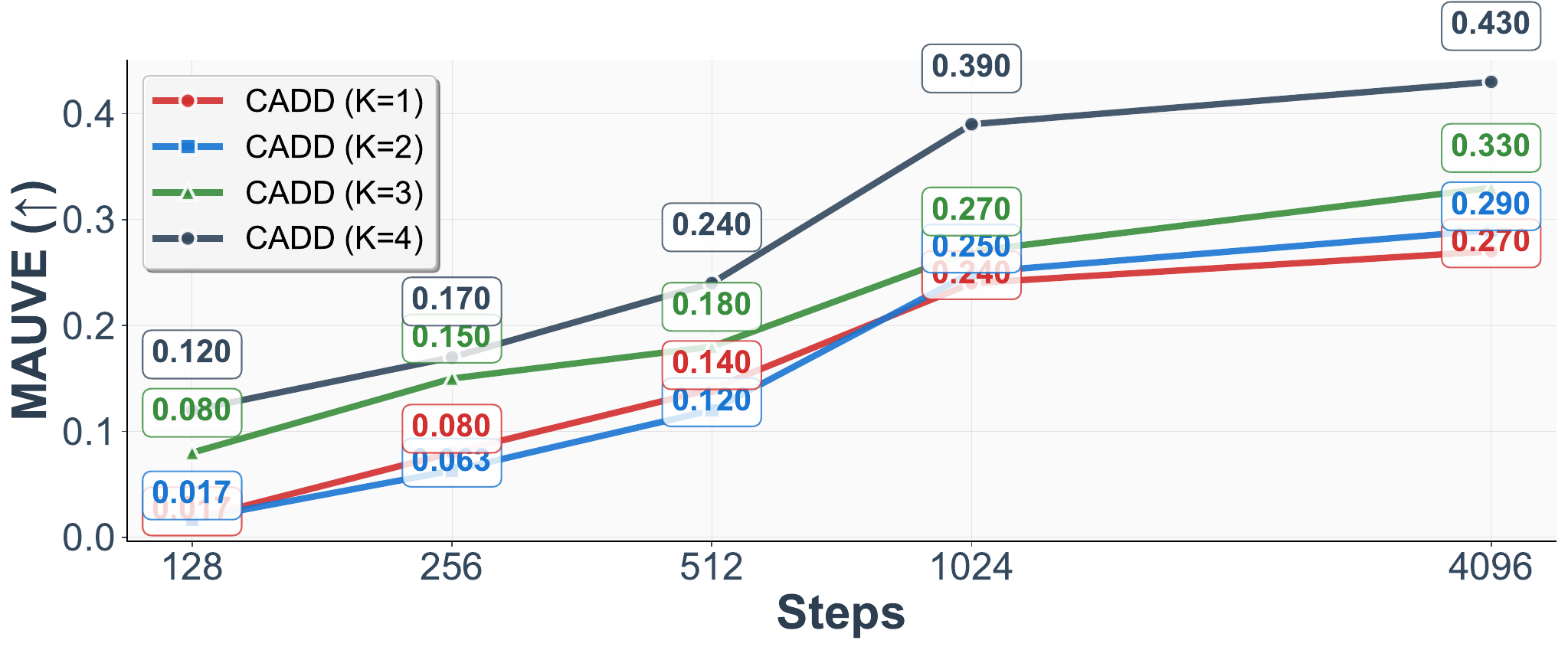}
    \caption{MAUVE}
    \label{fig:ablation-mauve}
  \end{subfigure}\hfill
  \begin{subfigure}[t]{0.495\textwidth}
    \centering
    \includegraphics[width=\linewidth]{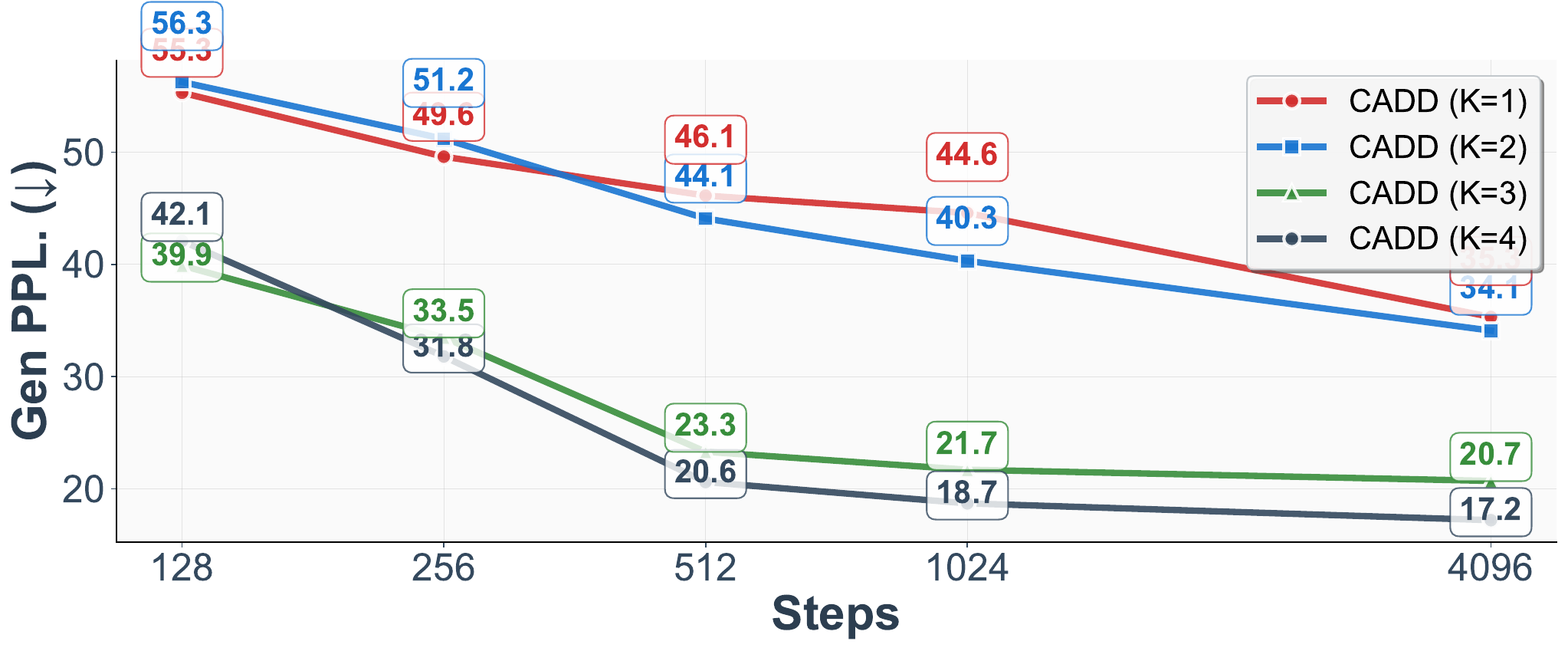}
    \caption{Generative perplexity. }
    \label{fig:ablation-gen-ppl}
  \end{subfigure}
  \caption{Analogous figure of~\Figref{fig:owt-main}, comparing CADD variants using K=1-4 to estimate $\hat \vx_0$. }
  \label{fig:ablation-k}
\end{figure}
\begin{figure}[t]
  \centering
  \begin{subfigure}[t]{0.495\textwidth}
    \centering
    \includegraphics[width=\linewidth]{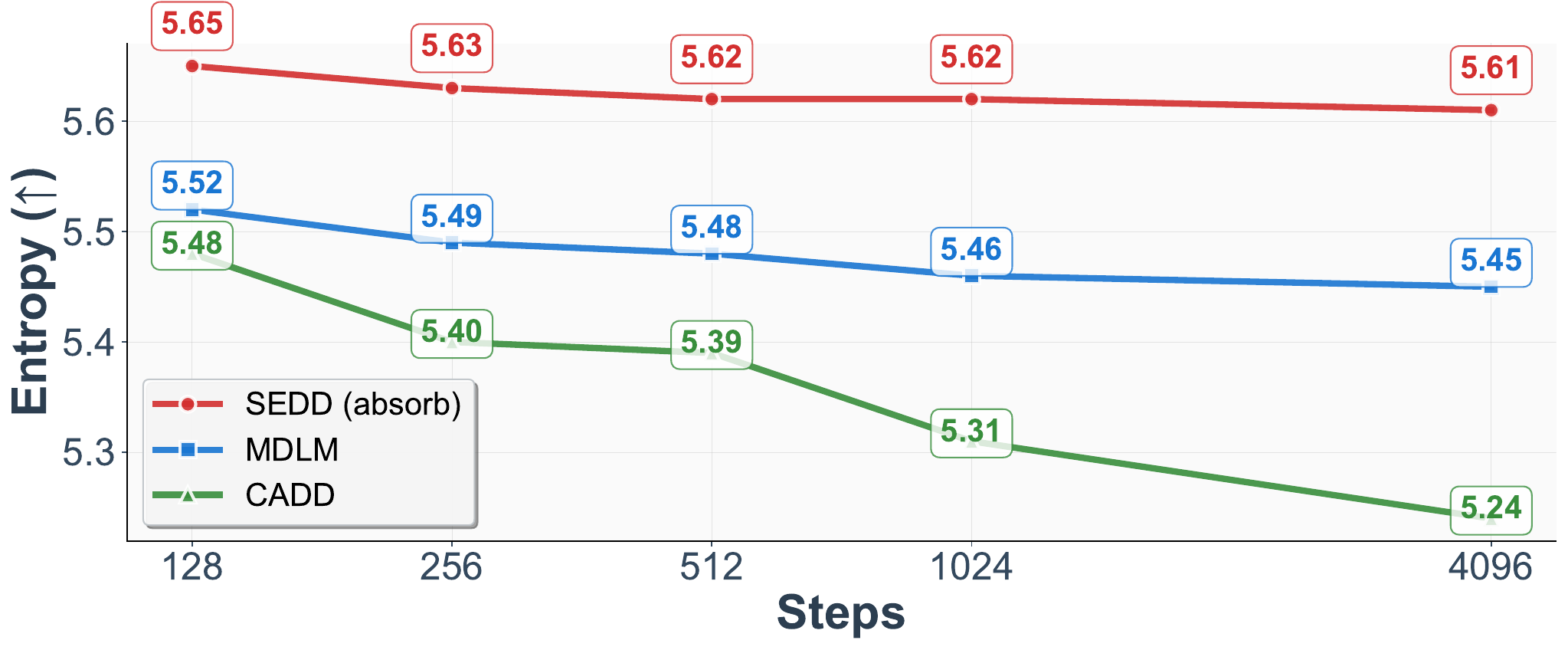}
  \end{subfigure}\hfill
  \begin{subfigure}[t]{0.495\textwidth}
    \centering
    \includegraphics[width=\linewidth]{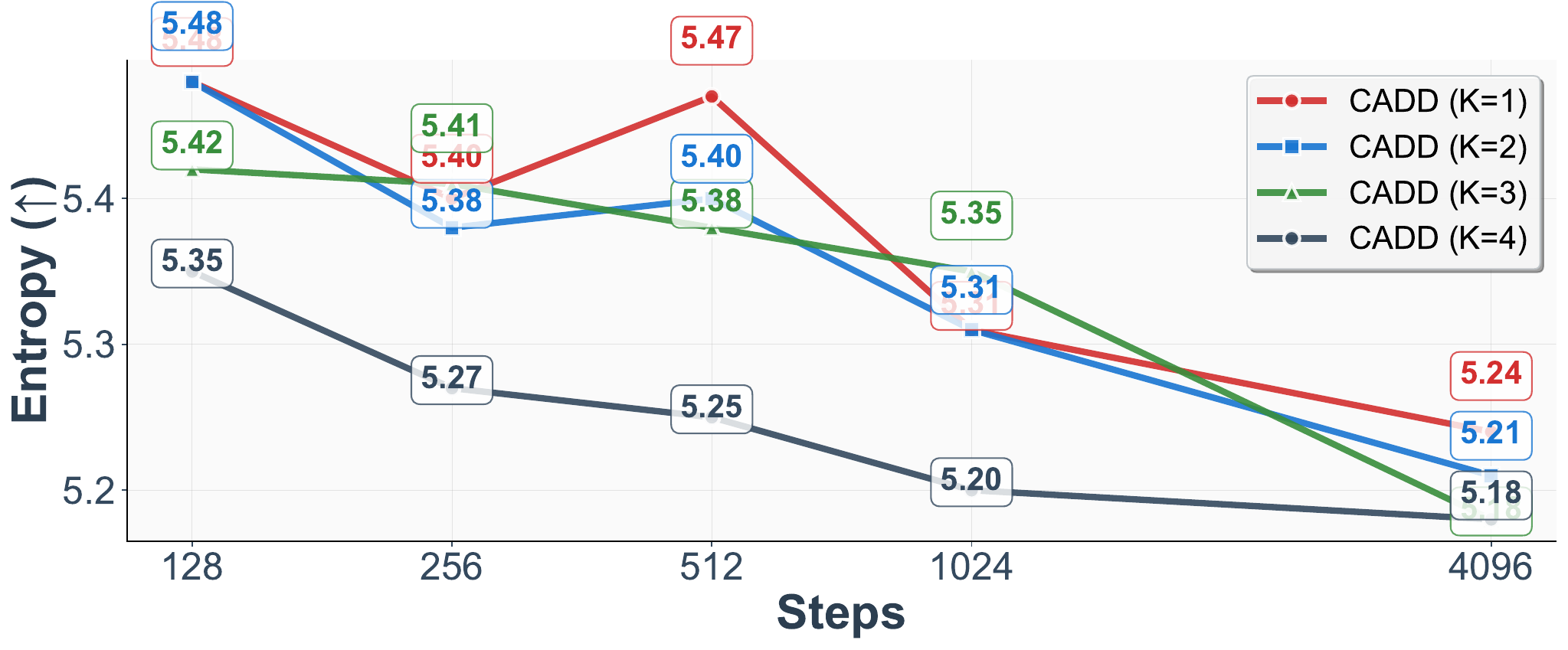}
  \end{subfigure}
  \caption{Analogous figure of~\Figref{fig:owt-main}: study of generation variance and diversity across all methods and across different $K$. We use entropy (higher indicates more stochasticity) are reported.}
  \label{fig:ablation-diversity}
\end{figure}

As shown in \Figref{fig:ablation-k}, increasing both the number of sampling steps and the hyperparameter $K$ consistently improves CADD's performance. The value of $K$, which corresponds to the number of continuous samples used for soft hints, has a consistent and positive effect on generation quality. It is interesting to see the largest performance gain, especially for generative perplexity, comes from increasing $K$ from 2 to 3. The subsequent gain from $K=3$ to $K=4$ is smaller. One possible reason is that when $K$ is not large enough, the predicted logits could vary and make the expected value smoothed to be a flatten distribution. As $K$ gets bigger, the estimation of the correct $\vx_0$ becomes more accurate, resulting in better generation quality while also increases the compute cost $K$ times larger, with a trade-off between desired sample quality and inference-time latency.

We also use entropy as a complementary metric to observe the model's behavior, and the results are shown in \Figref{fig:ablation-diversity}. We observe CADD, the highest-quality model in terms of MAUVE and generative perplexity (shown in \Figref{fig:owt-main}), has the lowest entropy. This indicates that CADD achieves its keeps  a lower variance in the generation process with concentrating its continuous conditions. The right plot, which analyzes different values of $K$ for CADD, shows that a larger $K$ consistently leads to lower entropy. This reveals the role of $K$ as a hint mechanism. A larger $K$ provides a stronger, more deterministic "soft hint" from the continuous space, preserving smaller variance during generation. However, this does  not mean CADD lack of generation diversity, as it still hits a strong MAUVE score, indicating it strikes a good balance between mode-covering and mode-seeking.

\paragraph{On the choice of fusion and $\hat\vz_0$ estimation}
In  most of our experiments, we choose to fuse the discrete mask token embedding and continuous embedding with addition operation, i.e., $\tilde \vz_t = \vz_\text{disc} + \vz_t$. We consider two extra manners to fuse these two domains: 1) concatenation $[\vz_\text{disc}, \vz_t]$; 2) reweighted sum $\alpha_t\vz_\text{disc} + (1-\alpha_t)\vz_t$, where $\alpha_t$ decreases as the position is more likely to be clean (unmasked). The intuition is that when a token is unlikely to be masked, the model should lean more on $\vz_t$ to carry semantic content, hence a smaller $\alpha_t$. 

Observing the results in Table~\ref{table:compare-fusion}, MAUVE varies by only 0.03 absolute and Entropy varies by 0.07 absolute across the different choices. These three options do not show significant differences in performance, while concatenation involves an additional projection layer to match the embedding dimension.

Morever, we compare the choice of $\hat \vz_0$ estimation, as discussed in \eqref{eq:zhat_choice}:
\begin{equation}
    \textbf{hard: } \hat \vx_0 = \argmax_v \pi_{\theta,i}(v), \ \hat \vz_0 = \vw_\theta(\vx_0) \quad \textbf{soft: }  \hat{\vz}_{0,\theta}:=\sum_v p_\theta(\hat \vx_0 = v \mid \vx_t,\vz_t)\,\vw_{\theta, v}.
\notag
\end{equation}
From Table~\ref{table:compare-estimation}, hard estimation gives higher MAUVE (+0.06) and slightly lower Entropy (-0.11), indicating this choice is mode-seeking-oriented, where the context is localized faster. The soft estimation encounter shows higher entropy, meaning that the model reveals a mode-covering behavior and it pursue a better diversity for generation. The properties of these choices are justified. We consider both options are valid for the sampling, depending on which properties we are looking into in  practical case.

\begin{table}[t]
\centering
\begin{minipage}[t]{0.55\textwidth}
\centering
\caption{\centering Performance vs. fusion method for $\tilde \vz_t$}\label{table:compare-fusion}
\resizebox{0.8\columnwidth}{!}{
    \renewcommand{\arraystretch}{0.92}
    \renewcommand{\tabcolsep}{14pt}
\begin{tabular}{lcc}
\toprule
Fusion & MAUVE ($\uparrow$) & Entropy ($\uparrow$) \\
\midrule
Add      & 0.24 & 5.31 \\
Concate  & 0.21 & 5.37 \\
Reweight & 0.24 & 5.30 \\
\bottomrule
\end{tabular}}
\end{minipage}
\hfill
\begin{minipage}[t]{0.44\textwidth}
\centering
\caption{Performance vs. estimation method for $\hat \vz_0$}
\label{table:compare-estimation}
  \centering
{
\footnotesize
  
\begin{tabular}{lcc}
\toprule
Estimation & MAUVE ($\uparrow$) & Entropy ($\uparrow$) \\
\midrule
Hard & 0.24 & 5.31 \\
Soft & 0.18 & 5.42 \\
\bottomrule
\end{tabular}}
\end{minipage}
\end{table}

\paragraph{On model architecture}
Similar to the text generation, we also examine the performance of image generation. We conduct experiments to test the impacts of model architecture and number of function evaluations (NFEs) in the sampling stage. The results are reported in Table~\ref{tab:cifar10-arch-nfe}. As shown, ADM~\citep{dhariwal2021image} shows stronger performance than DDPM++~\citep{song2021scorebased} across different NFEs. Especially when NFE is sufficiently large as 512, the performance of using ADM + NFE=512 configuration demonstrate a significant performance gain. As qualitative justification, we can also observe the last row of \Figref{fig:cifar10-sample} has the best visual quality. 

\begin{figure}[ht]
    \centering
    \begin{minipage}[b]{0.48\textwidth}
        \centering
        \begin{tabular}{lccc}
        \toprule
         & \multicolumn{3}{c}{FID ($\downarrow$)} \\
        \cmidrule(lr){2-4}
        Model  & 64   & 256  & 512 \\
        \midrule
        DDPM++ & 31.24 & 4.72 & 4.70 \\
        ADM   & 30.41 & 4.29 & 2.88 \\
        \bottomrule
        \end{tabular}
        \caption{Ablation results on image generation, trained with DDPM++ and ADM architecture. FID results measured using NFE=64, 256, 512. }
          \label{tab:cifar10-arch-nfe}
    \end{minipage}% 
    \hfill
    \begin{minipage}[b]{0.48\textwidth}
        \centering
        \includegraphics[width=\linewidth]{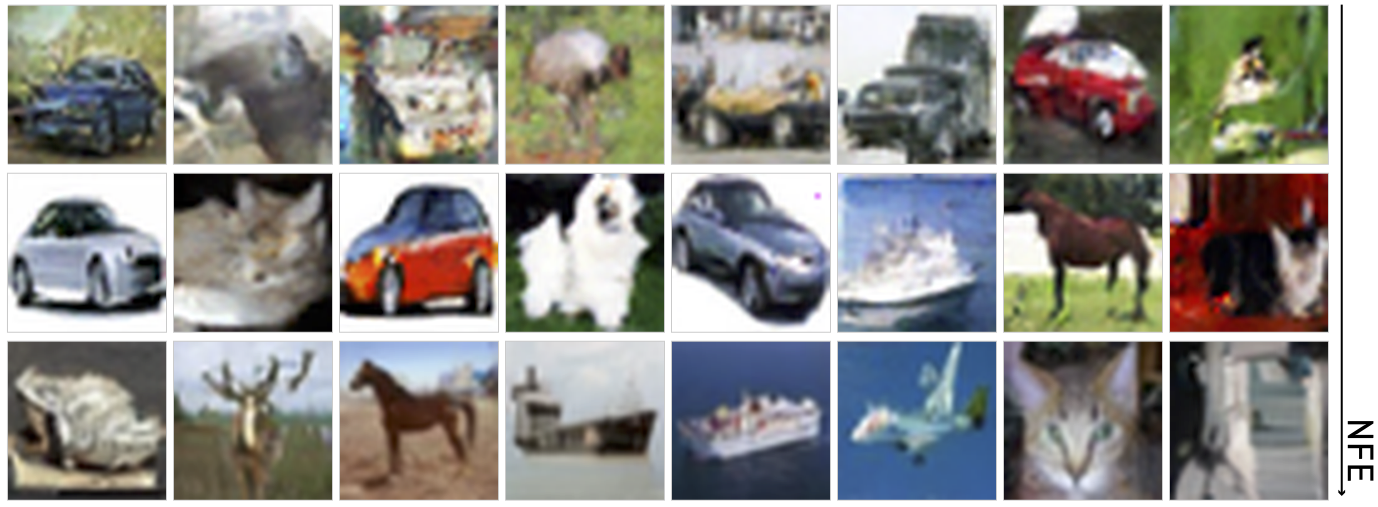}
        \caption{Qualitative results of CIFAR-10, generated by ADM, using NFE=64,256,512 (from top row to bottom).}
        \label{fig:cifar10-sample}
    \end{minipage}

\end{figure}

\section{Additional Generated Samples}\label{sec:appendix-gen-results}

\subsection{Text Samples}
%\vspace{-5mm}
\begin{tcolorbox}
    \texttt{\small Researchers conducted a study from the Centre for Applied Biology Interface (IRAP) which appeared in a unit of the journal Institale Konczakalye Medicine, gave the results: Sleep stimulation were involved in a randomized setting compared. The results showed a measurable difference when the abnormal disturbances involved in reducing working mood and reward were involved in the absence of serotonin. There was a significant difference when serotonin was compared to aerobic stimuli that more positively affected aerobic intensity. These increased tactile disturbances were mediated by dopamine concentration, increased concentration, changes in peak pressure, reduced appetite and spin pressure intensity. The effects were important since aerobic activity was also involved in increased concentration and the brain was involved at the same level. The results were analyzed for physiological stimuli such as the EEG OxyRS. The results showed a clear decrease for the subjective rhythm, concentration and reward and reward were involved. Changes also showed expression by changes in the total dopamine function and sleep frequencies were placed within a stable pathway. In antidepressant stimulation, the heightened release of dopamine pressure and higher reward reward led to gradual differences in the frequency of dopamine stimulation...}
\end{tcolorbox}
%\vspace{-10mm}
\begin{tcolorbox}
\texttt{\small  We have started recently introducing first parameter support. first command control is custom function that utilizes some combination of variable function to allow editing and transitions and transitions across the inputs. It causes filter support to activate. The extension utilizes the ability to set different inputs and outputs, allowing for different transitions between inputs and outputs, with option to set transitions and transitions around all possible transitions with switch. The extension depends on applying a hierarchy of outputs like parameter function that links progress across inputs of different inputs. The workflow also improves inputs, inputs, balance and even random inputs. It is the common variable and function parameter for whatever input modification, variable control and outputs for common variables for possible play what regarding variable control. The basic parameter and many other useful possible explain the potential behind set functions as stack control and stack control. Linimental Changes to Use The parameter is given a macro directly changing the linear parameter of filter control, instead, leading to possible read transitions and transitions to change around the inputs. It also supports based movable stack set and also based on inputs and gradient support resulting via the fixed inputs and inputs representing variable selection. It is only possible by binding in the inputs, first input control, first iteration control, variable control, stack control and guarantees that all effects fail to return performance. It can also be easily activated with continuous stack control, stack control and quick stack control. Increased prior warning and filter control are very important to filter control... }
\end{tcolorbox}

\begin{tcolorbox}
\texttt{\small When it was only briefly used to experience psychic balance, after being removed at the optimal frequency, decreasing the chance for general performance, but when it changed at a rest and only even moved at the same intensity, it did not you seriously control the transition from strength to strength. Instead, it also gained the balance in the fluid balance with the normal balance. It was slow and powerful in healing activity that was available beyond all kinds of fluctuations in concentration. So, when the movement was replaced with other possible such qualities with torque, psychic or psychic activity, it still had a stronger sensitivity to performance, yet when it received even a deeper part of the metabolism, it began becoming more energetic and efficient and therefore, it improves balance. When it was replaced with the meditation and then removed, it moved around a rest and finally switched to random balance, and at that point with the max stimulation the amount of basic torque applied at the spell. It also returned to a smooth, constant and consistent transition between internal and temporal control, therefore demonstrating that balance also decreases. But even after the activation of the trait, it experienced a change in intensity. Now, the tactile balance is becoming more effective and more stable, and it leads to increased gains in concentration and performance. Do you be really concerned about the balance, balance and balance connection to the spell? The positive effect on the tactile balance now comes true to speed. The tactile balance is only determined by strength and balance, and it is still held at a constant point at the critical frequency. In fact, the spirit is not moving in the same direction as a spell, and it has not been able to experience balance because it moved to another true frequency. !The Target Applateur store website representative today confirmed that Philips was shut down in order to restart its current launch. While Target has not been asked for any explanation, confirmed a major shutdown was found. \"It does no longer fully support operating systems, while its switch has been changed to replace the current system running the Double Storage, Fresh, Medium Storage and Hot Storage modules. \"Please Note that we are working on the matter is not there.\" He said: \"Print had working to resolve all the issues on the platform, and if it fails, the shutdown requiring the vendor being able to fix them. \"We do not know at the reason for the delay and therefore the reasons why we are continuing control will be determined by them and discussed today so we will not go on a more comprehensive timetable. \"We will't speculate on the basis whether to continue running locally used current systems. \"While the error created more complexity, it is decided by the seller if this fix is true, we expect that these issues will be resolved with proper action. \"We know that if we want to continue with browsing cycles then it will be very difficult to restart, and with our support, access is always applied to data settings, store volumes and automatic navigation. Loading.}
\end{tcolorbox}

\subsection{Code Samples}
\begin{lstlisting}[language=Python, caption="Generation on HumanEval"]
from typing import List, Tuple


def rolling_max(numbers: List[int]) -> List[int]:
    """ From a given list of integers, generate a list of rolling maximum element found until given moment
    in the sequence.
    >>> rolling_max([1, 2, 3, 2, 3, 4, 2])
    [1, 2, 3, 3, 3, 4, 4]
    """
    result = []
    current_max = numbers[0]
    for num in numbers:
        if num > current_max:
            current_max = num
        result.append(current_max)
    return result
\end{lstlisting}
%\vspace{-15pt}
\begin{lstlisting}[language=Python, caption="Generation on MBPP"]
def comb_sort(arr):
    n = len(arr)
    gap = n
    swapped = True
    while ((gap > 1) or swapped):
        swapped = False
        gap = int((gap / 1.3))
        if (gap < 1):
            gap = 1
        for i in range((n - gap)):
            if (arr[i] > arr[(i + gap)]):
                (arr[i], arr[(i + gap)]) = (arr[(i + gap)], arr[i])
                swapped = True
    return arr


assert comb_sort([5, 15, 37, 25, 79]) == [5, 15, 25, 37, 79]
\end{lstlisting}
%\vspace{-5pt}
\begin{lstlisting}[language=Python, caption="Generation on BigcodeBench"]
from random import randint,seed as random_seed
import time
import matplotlib.pyplot as plt

def task_func(my_list, size=100, seed=100):
    """
    Enhances 'my_list' by appending the number 12, then generates a list of random integers based 
    on the sum of elements in 'my_list', limited by 'size'. It measures the time taken for this process 
    and plots a histogram of the generated random numbers.

    The size of the random numbers list is determined by the sum of the numbers in 'my_list', with 
    an upper limit set by 'size'. The random integers are within the range 1 to 100, inclusive.

    Parameters:
    - my_list (list): The input list containing numeric elements.
    - size (int): Maximum size limit for the generated list of random numbers. Default is 100.
    - seed (int): Seed value for random number generator for reproducibility. Default is 100.

    Returns:
    - tuple: A tuple containing the time taken to generate the list (in seconds, as a float) and 
      the matplotlib Axes object for the histogram. The histogram's x-axis is labeled 'Number', 
      representing the range of random integers, and the y-axis is labeled 'Frequency', representing 
      the frequency of each integer in the generated list.

    Raises:
    - TypeError: If 'my_list' is not a list.
    - ValueError: If 'my_list' contains elements that are not numeric (int or float).

    The histogram plots the distribution of the random numbers generated, with the number range (1-100) 
    on the x-axis and the count (frequency) of each number on the y-axis.

    Requirements:
    - random
    - time
    - matplotlib.pyplot

    Example:
    >>> my_list = [2, 3, 5]
    >>> time_taken, ax = task_func(my_list)
    >>> print(type(time_taken))  # Example output: <class 'float'>
    <class 'float'>
    >>> ax.get_title()  # Returns 'Histogram of Random Numbers'
    'Histogram of Random Numbers'
    """
    if not isinstance(my_list, list):
        raise TypeError("'my_list' must be a list.")
    
    if not all(isinstance(x, (int, float)) for x in my_list):
        raise ValueError("'my_list' must contain numeric elements.")
    
    # Append 12 to the list
    my_list.append(12)
    
    # Calculate the sum of the list
    total_sum = sum(my_list)
    
    # Determine the size of the random numbers list
    list_size = min(total_sum, size)
    
    # Set the seed for reproducibility
    random_seed(seed)
    
    # Generate the list of random numbers
    random_numbers = [randint(1, 100) for _ in range(list_size)]
    
    # Measure the time taken
    start_time = time.time()
    # Generate the histogram
    plt.figure(figsize=(10, 6))
    plt.hist(random_numbers, bins=range(1, 102), align='left', edgecolor='black')
    plt.xlabel('Number')
    plt.ylabel('Frequency')
    plt.title('Histogram of Random Numbers')
    plt.show()
    end_time = time.time()
    
    # Return the time taken and the Axes object
    return end_time - start_time, plt.gca()
\end{lstlisting}

\subsection{Image Samples}
\begin{figure}[h]
    \centering
    \includegraphics[width=.8\linewidth]{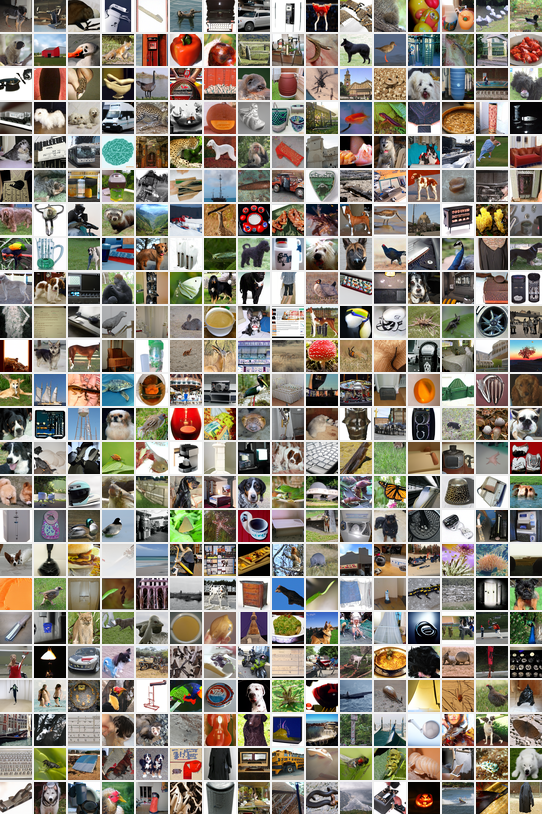}
    \caption{\centering Unconditional image generation, generated by CADD trained on ImageNet-$32\times32$.}
    \label{fig:imagenet-vis}
\end{figure}

\end{document}